\documentclass{article}

\usepackage{microtype}
\usepackage{graphicx}
\usepackage{subfigure}
\usepackage{booktabs}
\usepackage{amsmath}    
\usepackage{amssymb}  
\usepackage{bm}         
\usepackage{graphicx}
\usepackage{booktabs}
\usepackage{multirow}
\usepackage{enumitem}
\usepackage{soul}
\usepackage[numbers]{natbib}

\usepackage{amsthm}
\usepackage{mathtools}
\DeclarePairedDelimiterX{\norm}[1]{\lVert}{\rVert}{\mathopen{}\mathclose\bgroup#1\egroup}
\newtheorem{theorem}{Theorem}[section]

\newtheorem{lemma}[theorem]{Lemma}

\usepackage{caption}
\usepackage[T1]{fontenc}
\usepackage[scaled=0.92]{inconsolata}
\usepackage{xcolor}
\usepackage{listings}
\usepackage{svg}

\definecolor{plasmaA}{HTML}{6A00A8}
\definecolor{plasmaB}{HTML}{B12A90}
\definecolor{plasmaC}{HTML}{E16462}
\definecolor{plasmaBG}{rgb}{0.98,0.98,0.98} 

\newcommand{\highlight}[1]{\textbf{\textcolor{plasmaB}{##1}}}
\newcommand{\highlightA}[1]{\textbf{\textcolor{plasmaA}{##1}}}
\newcommand{\highlightB}[1]{\textbf{\textcolor{plasmaC}{##1}}}

\lstdefinestyle{pyplasma}{
  language=Python,
  basicstyle=\ttfamily\small,
  numbers=left,
  numberstyle=\color{black}\scriptsize, 
  numbersep=7pt,                        
  xleftmargin=1.5em,
  showstringspaces=false, upquote=true,
  columns=fullflexible, keepspaces=true,
  breaklines=true, breakatwhitespace=false, breakindent=1.5em,
  postbreak=\mbox{\textcolor{gray}{$\hookrightarrow$}\space},
  literate=* {.}{{.}}1 {(}{{(}}1 {)}{{)}}1 {[}{{[}}1 {]}{{]}}1
            {/}{{/}}1 {=}{{=}}1 {,}{{,}}1,
  alsoletter={_},
  moredelim=**[is][\bfseries\color{plasmaB}]{\\highlight\{}{\}},
  moredelim=**[is][\bfseries\color{plasmaA}]{\\highlightA\{}{\}},
  moredelim=**[is][\bfseries\color{plasmaC}]{\\highlightB\{}{\}},
}
\newcommand{\TwIST}{\texttt{TwIST}}

\usepackage[hyphens]{url}
\usepackage{hyperref}
\usepackage{float}

\usepackage[accepted]{mlsys2025}

\makeatletter
\renewcommand{\Notice@String}{} 
\makeatother

\mlsystitlerunning{TwIST: Rigging the Lottery in Transformers}

\begin{document}
\twocolumn[
\mlsystitle{TwIST: Rigging the Lottery in Transformers with \\ Independent Subnetwork Training}



\mlsyssetsymbol{equal}{*}
\begin{mlsysauthorlist}
\mlsysauthor{Michael Menezes}{rice}
\mlsysauthor{Barbara Su}{rice}
\mlsysauthor{Xinze Feng}{rice}
\mlsysauthor{Yehya Farhat}{rice}
\mlsysauthor{Hamza Shili}{rice}
\mlsysauthor{Anastasios Kyrillidis}{rice}
\end{mlsysauthorlist}

\mlsysaffiliation{rice}{Department of Computer Science, Rice University, Texas, USA}

\mlsyscorrespondingauthor{Anastasios Kyrillidis}{anastasios@rice.edu}

\mlsyskeywords{Machine Learning, MLSys}

\vskip 0.3in

\begin{abstract}
We introduce \TwIST{}, a novel distributed system for efficient Large Language Model (LLM) training. Motivated by our ``\textit{golden lottery ticket hypothesis},'' \TwIST{} trains subnetworks in parallel, periodically aggregating and resampling, yielding high-performance subnets (``golden tickets'') that require no fine-tuning. This enables robust, zero-cost pruning at deployment, achieving perplexity scores close to state-of-the-art post-training methods while bypassing their post-training overhead (e.g., calibration, Hessian inversion). \TwIST{}'s advantage emerges under aggressive pruning (e.g., 50\%+ sparsity), where it significantly outperforms baselines; for example, achieving 23.14 PPL while the closest baseline follows at 31.64. As a \textit{structured pruning} method, \TwIST{} produces smaller, dense matrices, translating to tangible inference speedups and memory savings on commodity hardware deployments (e.g., CPUs) that lack sparse computation support. We provide the complete implementation \href{https://anonymous.4open.science/r/twist2-373F}{here}.
\end{abstract}
]

\printAffiliationsAndNotice{}  

\section{Introduction}

Large Language Models (LLMs) \citep{achiam2023gpt, brown2020language} have reshaped the field of AI with their performance across a wide range of tasks. The Generative Pretrained Transformer (GPT) family has proven to be a powerful architecture, demonstrating generalization on diverse and complex language benchmarks \citep{bommarito2022gpt, chen2021evaluating, wei2022emergent}. However, training and deploying these models come with massive computational costs. For instance, DeepSeek-V3 has approximately 671 billion parameters. Storing such a model with half-precision floating-point numbers (FP16) would require around 1.34 TB of memory ($2$ bytes per parameter $\times$ $671\text{B}$ parameters). To put this into perspective, even for holding the parameters alone, this would necessitate about 17-20 NVIDIA A100 GPUs, each equipped with 80 GB of memory; real inference typically would need 25-100 A100s depending on context length and KV precision. Thus, to democratize the use of LLMs, the research community has explored powerful techniques to mitigate these resource bottlenecks, primarily through model compression. Two of the most prominent approaches are quantization and pruning. 

Quantization reduces the memory footprint of a model by representing its parameters with lower-precision numerical formats \citep{dettmers2022gpt3, frantar2022gptq, ahmadian2023intriguing, yao2022zeroquant}. This technique is a cornerstone of model compression, with two main strategies: Post-Training Quantization (PTQ) and Quantization-Aware Training (QAT). PTQ offers a straightforward way to compress a pre-trained model but could sometimes lead to a non-neglibigle decline in accuracy. On the other hand, QAT simulates the quantization process during training, often preserving higher accuracy at the cost of increased training complexity and computational overhead. A commonality in these approaches is that the expensive full-precision full-model training rounds could still often be prerequisite.

A complementary approach to quantization is \textit{pruning} \citep{lecun1989optimal, hassibi1993optimal}, which removes individual weights from the model. This approach is supported by concepts like the Lottery Ticket Hypothesis (LTH) \citep{frankle2018lottery}, which claims that dense networks contain sparse ``winning tickets'' that can match full model accuracy. Despite LTH, the wide adoption of pruning for LLMs remains an open challenge, largely because finding these subnetworks is difficult. Some of the more successful methods require extensive retraining \citep{liu2018rethinking, blalock2020state} or costly iterative and fine-tuning procedures \citep{frankle2018lottery, renda2020comparing}. In general, although other sparsity inducing \citep{evci2020rigging, sanh2020movement} or pruning-aware \citep{han2015learning, liu2021sparse} training regimes have shown some moderate success,  they still require multiple training passes and extensive amounts of memory. The more practical pruning techniques for LLMs are post-training pruning (PTP) methods, which compress an already trained model without any re-training \citep{bhuiyan2025z, frantar2023sparsegpt, sun2023simple}. While PTP is computationally less demanding, it often still involves solving complex, and usually expensive subproblems. A common thread among these pruning strategies is that they require full-model training iterations, attempting to find important weights after this step. 

It is unclear if the performance of current pruning algorithms represents an upper bound on the quality of sparse models. \citet{gale2019state} found that three different post-training methods all achieve about the same sparsity / accuracy trade-off. Thus, it is an open question whether better performance trade-offs are possible. \citet{prasanna2020bert} study the LTH from the perspective of a pretrained BERT model \citep{devlin2019bert} and empirically confirm the existence of winning tickets. But their results also reveal other interesting insights. The same non-retrained winning lottery tickets of these models are actually not that far behind the retrained tickets. Furthermore, their work hints at the surprising viability of random pruning. While a randomly pruned and then retrained subnet often lags behind a winning ticket, its performance is not negligible. Such random pruning at initialization \citep{su2020sanity, liu2022unreasonable, gadhikar2023random} is almost always favored for its simple, computationally cheap, and data independent nature.

Inspired by \citep{prasanna2020bert} and focusing on Transformer-based neural network training, if the main bottleneck is the expensive, iterative search for a specific ``winning ticket,'' and not the retraining, we ask whether we can redesign the training process itself to eliminate the search. I.e., what if, instead of creating a dense model with a few high-performing subnets, we could train a model where high performance is the default for most subnets?

This leads us to propose an extension of the LTH, which we term the \textit{golden lottery ticket hypothesis}: \textbf{\textit{A dense neural network can be trained such that the vast majority of its randomly sampled subnets achieve high performance without any subsequent training or further fine-tuning.}} Such a model would be inherently compressible and efficient, as sparsity could be achieved by simple random selection rather than a costly search algorithm. From here we ask: how can we regularize a model during training such that nearly any randomly sampled subnet is a ``golden ticket?''

The practical benefits of such an approach could be significant. Consider a scenario with a diverse ecosystem of end-user devices, from high-end servers to everyday smartphones. With an inherently compressible parent model, we could deploy a spectrum of smaller ``tickets'' tailored to the computational capabilities of each device. This would enable a consistent user experience across different hardware, with each device running a version of the model that is not just smaller, but also comparably proficient. This paradigm shifts away from a one-size-fits-all deployment strategy to a more flexible and efficient distribution of AI.

\textbf{Our approach and contributions.} We sidestep the expensive search-and-retrain paradigm of pruning and instead introduce \textit{Transformers with Independent Subnetwork Training (\TwIST{})}, a novel distributed training algorithm that is inspired by \citet{yuan2019distributed, wolfe2024gist, dun2023efficient, hu2023federated,  dun2022resist}. The algorithm is efficient compared to other standard distributed algorithms and designed to imbue the model with an inherent structural robustness. By training independent subnets across different compute nodes, where each subnet spans the same number of layers as the original mode, \TwIST{} encourages the full model to develop a weight structure where multiple pathways are inherently performant. This approach tackles training and deployment efficiency simultaneously, aiming to produce a model that is compressible by default.
The key contributions of our work can be summarized as follows: 
\begin{itemize}
    \item We introduce the \textit{golden lottery ticket hypothesis}: that a dense network can be trained so that randomly sampled subnets achieve high performance \textit{without fine-tuning}.
    \item We empirically validate our hypothesis on text generation, using \TwIST. Demonstrating the feasibility of zero-cost pruning at deployment. As shown in Table~\ref{tab:subnet_performance}, \TwIST{} is highly competitive with SOTA methods while incurring low post-training overhead (e.g., calibration, Hessian inversion).
    \item We show that \TwIST's subnets induce both system stability and architectural robustness, making it suitable for fault-tolerant model-parallel deployments (See Figures~\ref{fig:subnet_loss_distribution_attn} and \ref{fig:robustness_heatmap_attn}).
    \item We highlight that in aggressive structured pruning scenarios \TwIST{} excels (e.g. at a 4/12 ratio, \TwIST{} achieves 23.14 PPL while the closest PTP baseline follows at 31.64) and unlike unstructured methods we see tangible speedups on commodity hardware.
\end{itemize}

\section{Background}

\textbf{Notation.} Vectors and matrices are represented with bold font (e.g., $\bm{x}$), while scalars by plain font (e.g., $x$ or $S$). Capital letters distinguish matrices from vectors (e.g., $\bm{W}$ vs $\bm{w}$). Calligraphic uppercase letters denote sets (e.g., $\mathcal{D}$); the cardinality of $\mathcal{D}$ is represented as $|\mathcal{D}|$. 

\textbf{Problem formulation.} We consider a distributed training setup over $S$ compute nodes, where node $s$ holds local data $\mathcal{D}_s = \{(\bm{X}_i, \bm{Y}_i)\}_{i=1}^{|\mathcal{D}_s|}$. We assume each local dataset $\mathcal{D}_s$ is drawn independently and identically distributed (i.i.d.) from a global data distribution. Our goal is to train a transformer-based model, whose parameters are collectively denoted
\begin{equation*}
\bm{W} = \{ \bm{W}^{\text{embd}}, \bm{W}^{\text{proj}} \} \cup \{ \bm{W}^{\text{layer}}_l \}_{l = 1}^{L}
\end{equation*}
where $\bm{W}^{\text{embd}}$ is the token embedding, $\bm{W}^{\text{proj}}$ is the task specific final projection, and each layer
\begin{equation*}
\bm{W}^{\text{layer}} = \{ \bm{W}^{\text{ln}}, \bm{W}^Q, \bm{W}^K, \bm{W}^V, \bm{C}^{\text{attn}}, \bm{W}^{\text{ffn}}, \bm{C}^{\text{ffn}} \}
\end{equation*}
contains the layer norm parameters ($\bm{W}^{\text{ln}} \in \mathbb{R}^{d_{\text{model}} \times 2}$); the query, key, value ($\bm{W}^Q, \bm{W}^K, \bm{W}^V \in \mathbb{R}^{d_{\text{model}} \times H d_{\text{head}}}$), and output ($\bm{C}^{\text{attn}} \in \mathbb{R}^{H d_{\text{head}} \times d_{\text{model}}}$) projection matrices for multi-head attention; and the feedforward network weights ($\bm{W}^{\text{ffn}} \in \mathbb{R}^{d_\text{model} \times d_{\text{inner}}}, \bm{C}^{\text{ffn}} \in \mathbb{R}^{d_\text{inner} \times d_{\text{model}}}$). Moreover, we define $d_{\text{model}}$ as the model embedding dimension, $d_{\text{head}}$ as the dimension of a single attention head, $H$ as the number of attention heads, and $d_{\text{inner}}$ as the feedforward hidden dimension. The goal is to find values for $\bm{W}$ that achieve good accuracy on all data $\mathcal{D}=\cup_s \mathcal{D}_s$, by minimizing the following optimization objective:
\begin{equation*}
    \bm{W}^\star \in \underset{\bm{W}}{\arg\min}~
    \Biggl\{\,\mathcal{L}(\bm{W})
    := \frac{1}{S}\sum_{s=1}^S
    \ell\!\left(\bm{W}_s, \mathcal{D}_s\right)
    \Biggr\},
    \label{eq:Twist_loss}
\end{equation*}
where $\ell\!\left(\bm{W}_s, \mathcal{D}_s\right) = \tfrac{1}{|\mathcal{D}_s|} \sum_{( \bm{X}_i, \bm{Y}_i ) \in \mathcal{D}_s} \ell\!\left(\bm{W}_s, ( \bm{X}_i, \bm{Y}_i ) \right)$. Here, $\ell\!\left(\bm{W}_s, \mathcal{D}_s\right)$ denotes the \textit{local} loss function for user $s$, associated with a local model $\bm{W}_s$, that gets aggregated with the models of other users. $\bm{W}_s$ is either a full copy of the global model at the current training round or a selected submodel of the global one. 

Traditional distributed training follows either \textit{data parallelism} \citep{farber1997parallel, raina2009large,li2020pytorch}, where each node trains the full model on local data, or \textit{model parallelism} \citep{dean2012large, huang2019gpipe}, where model layers or partitions are split across nodes.  While both optimize the same global objective, they suffer from communication bottlenecks: data parallelism requires synchronizing large dense models, and model parallelism involves fine-grained layer-level exchanges that are costly in practice. Furthermore, while achieving a global model $\widehat{\bm{W}} \approx \bm{W}^\star$ is the theoretical goal, a significant practical consideration remains: the deployability of this model. In many real-world scenarios, the compute nodes are resource-constrained edge devices where a large, dense model $\widehat{\bm{W}}$ would incur unacceptable latency and energy costs during inference. 

The conventional solution here is to treat model compression as a separate, post-training step. This involves taking the fully trained $\widehat{\bm{W}}$ and applying pruning techniques to obtain a pruned model $\widehat{\bm{W}}'$, where, with a slight abuse of notation $|\widehat{\bm{W}}| \gg |\widehat{\bm{W}}'|$. Such a process is notoriously expensive as it necessitates an iterative cycle of removing parameters which might require solving complex subproblems and/or extensive fine-tuning to regain the initial accuracy \citep{han2015deep, frantar2023sparsegpt, yang2025wanda++}. Moreover, it decouples the primary training objective from the ultimate goal of obtaining an efficient final model.

\section{Overview of \TwIST{}}



\begin{figure*}[htbp]
\centering
\includegraphics[width=\textwidth]{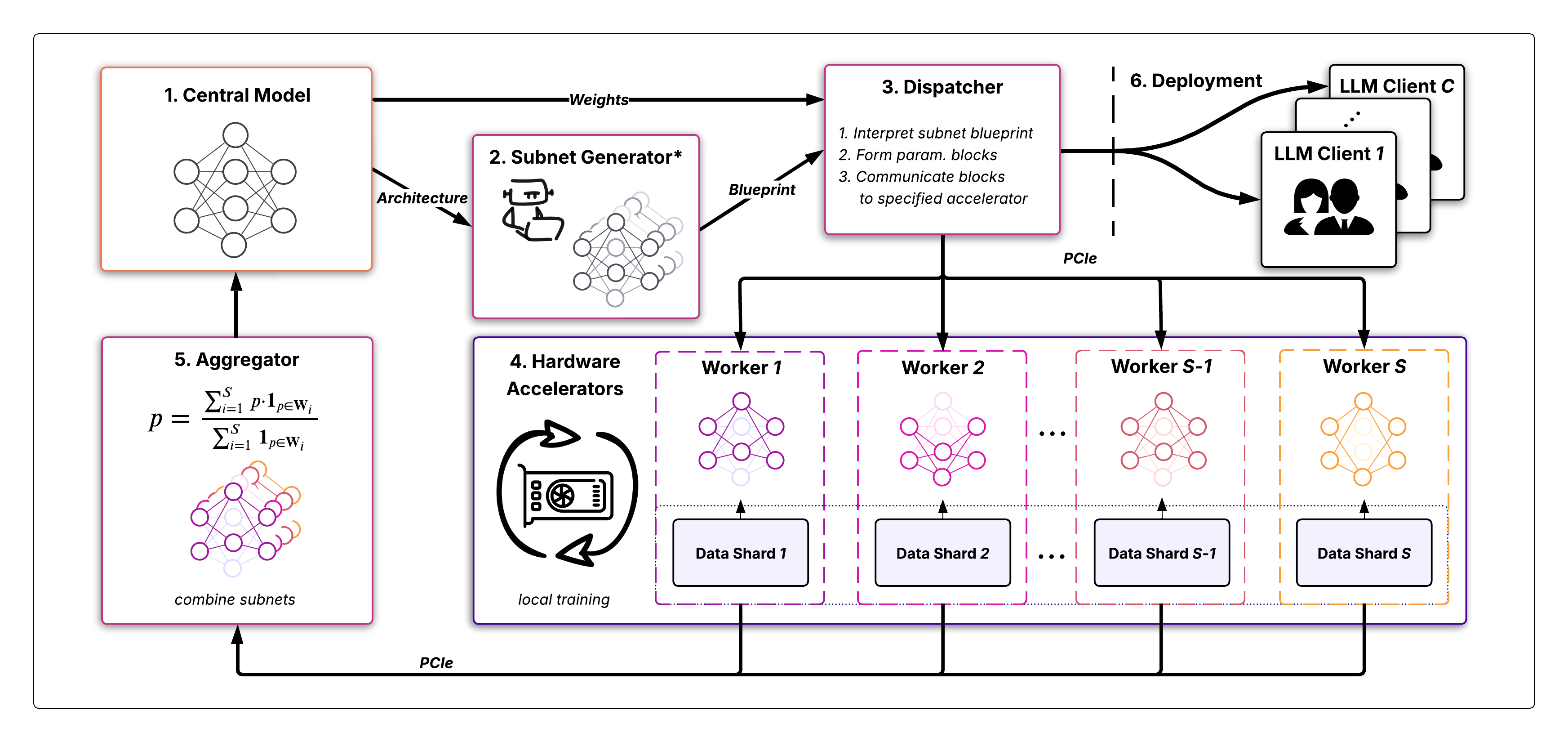} \vspace{-0.4cm}
\caption{
    \TwIST{} system overview. 
    \textbf{(1)} From a central model, \textbf{(2)} a subnet generator* creates diverse subnets. 
    \textbf{(3)} A dispatcher sends these subnets via Peripheral Component Interconnect express (PCIe) to \textbf{(4)} multiple workers for parallel training on distinct data shards. 
    \textbf{(5)} An aggregator updates the central model by averaging the parameters from the trained subnets using the shown formula. 
    \textbf{(6)} The final model is then deployed to LLM (Large Language Model) clients for inference. 
    (*The generator supports different heuristics for training vs. deployment.)
}
\label{fig:twist_system} \vspace{-0.4cm}
\end{figure*}

Figure \ref{fig:twist_system} presents an overview of the \TwIST{} system architecture. The design enables flexible and efficient model compression while also reducing cumulative training costs. \TwIST{} is comprised of three main components, which abstractly relates to previous work \citep{yuan2019distributed, wolfe2024gist, dun2023efficient, hu2023federated,  dun2022resist}: \textbf{\textit{(i)}} a \textbf{subnet generator} to create model blueprints based on the available edge devices,  \textbf{\textit{(ii)}} a \textbf{dispatcher} to materialize them on each edge device, and \textbf{\textit{(iii)}} an \textbf{aggregator} to update the central model on the server. For each communication round, we first generate a set of subnetwork blueprints, create and send a subnetwork to each respective device, and finally send back the updated subnetworks to the server for aggregation. This process is repeated until convergence. For deployment, the subnetwork generator is supplied with the target constraints to construct a blueprint of the desired model size before the dispatcher sends the pruned model to the target client (edge device).

\subsection{Subnet Generator} 
\label{sec:subnet_generator}

\begin{algorithm}[tb]
\caption{Generate Training Subnet Blueprint}
\label{alg:gen_subnet}
\begin{algorithmic}[1]
\REQUIRE $N_{\text{full}}$ (\# of full model blocks), $S$ (\# of subnets to generate), $N_{\text{sub}}$ (\# of blocks per subnet), $\mathcal{C}$ (set of block indices common to all subnets)
\ENSURE $\bm{A} \in \mathbb{N}^{S \times N_{\text{sub}}}$

\STATE \textcolor{violet}{\# Assign common blocks}
\STATE $\bm{A}[:, :|\mathcal{C}| ] = \mathcal{C}$

\STATE \textcolor{violet}{\# Ensure every block is assigned}
\FOR {$i, b \in \text{enumerate} (\mathcal{C'})$}
    \STATE $\bm{A}[i \bmod S, |\mathcal{C}| + \lfloor i / S \rfloor] = b$
\ENDFOR

\STATE \textcolor{violet}{\# Ensure every subnet has $N_{sub}$ blocks}
\FOR{$s \in \text{range}(S)$}
    \STATE $\mathcal{B}_{\text{filled}} \leftarrow \text{set}(\bm{A}[s, :])$
    \STATE $N_{\text{empty}} \leftarrow N_{\text{sub}} - |\mathcal{B}_{\text{filled}}|$
    \STATE $\bm{A}[s, -N_{\text{empty}}:] \leftarrow \text{unique\_choice} ( \mathcal{B}_{\text{filled}}', N_{\text{empty}} )$
\ENDFOR

\STATE \textbf{return} $[\text{sorted} (\bm{A}_s) \textbf{ for } \bm{A}_s \in \bm{A}] $
\end{algorithmic}
\end{algorithm}

For efficient subnetwork creation during training, the \textit{subnet generator} generates subnets uniformly at random while ensuring that every parameter is included in training. We focus only on the case of transformers that was missed by literature. Since subnets are formed by subsampling attention heads from the attention layers and neuron blocks from the feedforward layers, we henceforth use the term ``blocks'' to refer agnostically to either attention heads or feedforward layer neuron chunks. Put simply, a block is a chunk of parameters in memory. Formally, we represent the $l$-th attention layer with $H$ attention heads as $\bm{W}^{\text{attn}}= \{ \bm{W}^Q, \bm{W}^K, \bm{W}^V, \bm{C}^{\text{attn}} \}$. For the $l$-th feedforward layer, we define it as $\bm{W}^{\text{ffn}} = \{\bm{W}^{\text{ffn}}, \bm{C}^{\text{ffn}} \}$. We partition the attention layer into $H$ distinct blocks, where the $h$-th block of the attention layer is defined as $\bm{W}^{\text{AttnBlock}}_h = \{ \bm{W}^Q_h, \bm{W}^K_h, \bm{W}^V_h, \bm{C}^{\text{attn}}_h \}$ where $\bm{W}^Q_h, \bm{W}^K_h, \bm{W}^V_h, ( \bm{C}^{\text{attn}}_h )^\top \in \mathbb{R}^{d_{\text{model}} \times d_{\text{head}}}$. For the feedforward layer, we partition it into $R$ distinct blocks, where $R$ is a parameter that is defined by the user and influenced by the capacity of the edge device. The $r$-th feedforward block is defined as $\bm{W}^{\text{FfnBlock}}_r = \{\bm{W}^{\text{ffn}}_r, \bm{C}^{\text{ffn}}_r \}$ where $\bm{W}^{\text{ffn}}_r, ( \bm{C}^{\text{ffn}}_r )^\top \in \mathbb{R}^{d_{\text{model}} \times \frac{d_{\text{inner}}}{R} }$. All other parameters (e.g layernorm, token embedding, and final projection) are shared across subnetworks.

By interpreting our weights as a concatenation of these blocks, the problem of creating a subnetwork simplifies to choosing which blocks from the central model should be included in each respective subnetwork. Our work focuses on the case of workers with homogeneous compute (e.g., identical GPUs). Thus, we represent the block assignment blueprint for each layer of the network at every communication round as a matrix $\bm{A} \in \mathbb{N}^{S \times N_{\text{sub}}}$ where $S$ represents the number of workers and $N_{\text{sub}}$ represents the number of blocks in a subnetwork. One core idea in pruning is that some blocks are more critical than others for overall subnet performance \citep{michel2019sixteenheadsreallybetter, zheng2025dense2moe}. For this reason one can also define a set of blocks $\mathcal{C}$ that is common across all subnetworks. We note that in all our experiments this is not necessary, as \TwIST{} exhibits competitive results without the need for fixing parameters across subnetworks more than necessary (i.e.,  $\mathcal{C} = \emptyset$). We make an exception and share the entire first and last few layers to avoid severe performance drops \citep{kim2024shortened}.

Algorithm \ref{alg:gen_subnet} provides an overview of the subnet generation algorithm used during training based on the random heuristic. Every generated subnetwork must $i)$ contain the set of common blocks, $ii)$ every block is assigned to at least one subnetwork, $iii)$ and all subnetwork must satisfy  the size constraint $N_{\text{sub}}$. These constraints can be formally defined as two inequalities $N_{\text{sub}} \leq N_{\text{full}}$ and $N_{\text{full}} - |\mathcal{C}| \leq S (N_{\text{sub}} - |\mathcal{C}|)$. Taken together we get a bound on the size of our subnets
\begin{equation*}
\frac{N_{\text{full}} + (S - 1) |\mathcal{C}|}{S} \leq N_{\text{sub}} \leq N_{\text{full}}.
\end{equation*}
The above constraint only applies to subnetwork during training. When deploying, we relax our constraints and shift our focus to creating a single strong subnetwork.

\subsection{Dispatcher}

The \textit{dispatcher} is responsible for materializing the subnet generator's blueprint. For this work, we employ a standard server-client network topology where the full central model is hosted on the server. After a user configures the number of blocks, the dispatcher iterates through the transformer's layers and determines the specific weights and biases to chunk and the dimension along which to do so.

The shared weight matrices (e.g., word embedding and layer norms) are simply broadcasted. The remaining weight matrices (e.g., $\bm{W}^Q, \bm{W}^K, \bm{W}^V, \bm{C}^{\text{attn}}, \bm{W}^{\text{ffn}}, \bm{C}^{\text{ffn}}$) are broken up into blocks, and based on the subnet generator's blueprint are concatenated to form a submatrices before being scattered. This concatenation before scattering reduces the number of inter-worker communication rounds and is key to reaping the benefits in training latency demonstrated in Figure~\ref{fig:latency_ablation}.

\subsection{Aggregator}

The \textit{aggregator} is the counterpart of the dispatcher. It combines the partially trained subnets and updates the central model parameters. The $S$ subnets produced by \TwIST{} are mostly disjoint, meaning that most model parameters are not simultaneously partitioned to multiple subnetworks. Given that most parameters of the subnets are disjoint the aggregator copies the parameters back into the full central model, where no collisions occur. For the shared parameters we borrow the updated procedure of FedAvg \cite{mcmahan2017communication} and update the shared parameters by the average value across the subnetworks. Formally the updated value of a central model parameter $p$ is given by
\begin{equation*}
 p = \frac{\sum_{i = 1}^S p \cdot \bm{1}_{p \in \bm{W}_i}}{\sum_{i = 1}^S \bm{1}_{p \in \bm{W}_i}}
\end{equation*}
where $\bm{W}_i$ represents the $i$-th subnet's trainable parameters.

\subsection{Algorithmic Properties}

The pursuit of desirable algorithmic properties shapes the of design of \TwIST. We explore these choices here.


\textbf{Correcting Activation Shift.}
We have theoretically shown and empirically verified that subsampling blocks (i.e., attention heads or feedforward neurons) shifts the distribution of activations under common assumptions. In particular, when sampling $N_{\text{sub}}$ out of $N_{\text{full}}$ blocks from a layer, we observe a general relationship for the layer's output activation $\bm{y}$:
\begin{align*}
\mathbb{E} [\norm{\bm{y}'}] = \sqrt{\frac{N_{\text{sub}}}{N_{\text{full}}}} \mathbb{E} [\norm{\bm{y}}],
\end{align*}
where $\bm{y}$ and $\bm{y}'$ are the output activations of the full and subsampled layers, respectively. We scale subnet activations prior to the residual connection in the attention and feedforward layers by $\sqrt{\frac{N_{\text{full}}}{N_{\text{sub}}}}$ to counteract this effect. Derivations of the above relations are given in Appendix Theorem~\ref{thm:ffn_scale_factor} and Theorem~\ref{thm:attn_scale_factor}.


\begin{figure}[htbp]
\centering
\includegraphics[width=0.95\columnwidth]{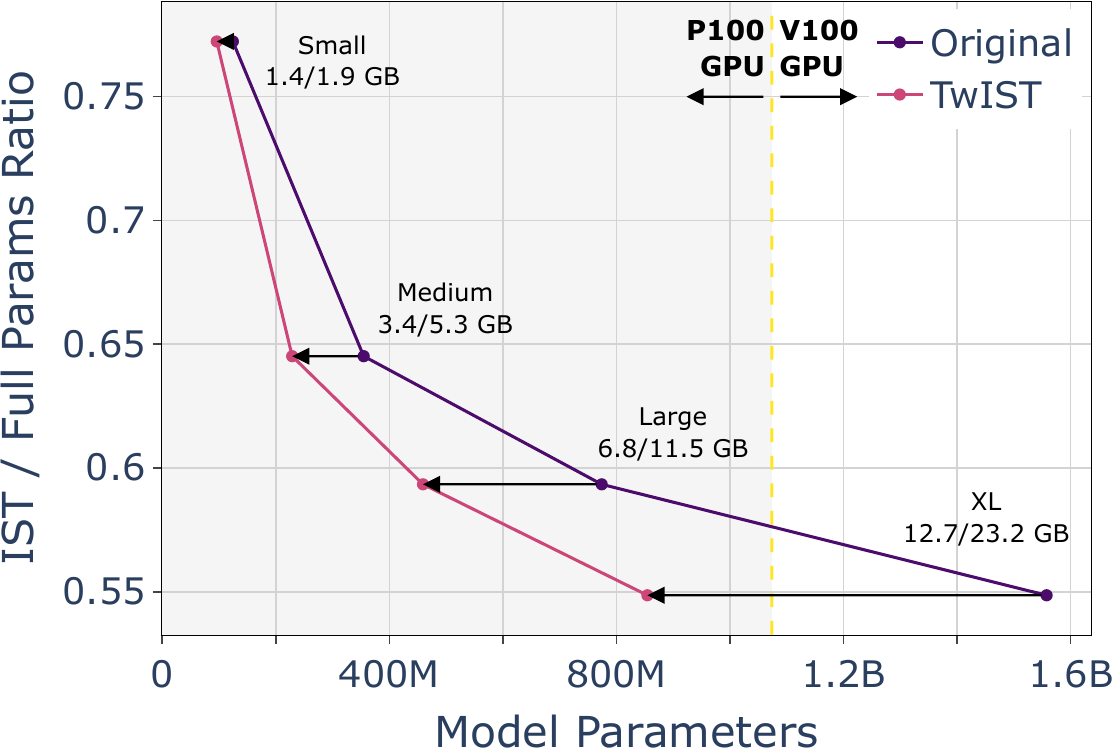}
\caption{\TwIST's asymptotic impact on memory for GPT-2 model variants. Subnets have half the blocks of the full model.}
\label{fig:asymptotic_memory}
\end{figure}

\textbf{Asymptotic Memory Ratio.}
The \TwIST{} method only partitions a subset of the parameters when forming models. Specifically, when constructing a subnet with $N_{\text{sub}}$ blocks out of a total of $N_{\text{full}}$ blocks per layer, the physical memory consumption on the hardware accelerator exceeds the ideal proportional ratio of $\tfrac{N_{\text{sub}}}{N_{\text{full}}}$ due to the inclusion of unpartitioned shared parameters (e.g., embeddings and normalization layers). However, an important architectural trend in large-scale transformers is the relative increase in the proportion of attention and feedforward network parameters compared to static parameters, such as token embeddings and layer normalizations. As demonstrated by our analysis on variants of the GPT-2 architecture (Figure \ref{fig:asymptotic_memory}), this trend leads to an asymptotic memory convergence: for increasingly large full models, subnets comprising half the total blocks ($\tfrac{N_{\text{full}}}{2}$) approach $\approx 50\%$ of the full model's total physical memory footprint. This pronounced asymptotic effect corroborates the model size reduction efficacy of memory-efficient techniques targeting sparsity in attention and feedforward layers.

\textbf{Exploration-Exploitation Trade-off.}
\TwIST{} builds on established methods like IST-family \citep{yuan2019distributed, wolfe2024gist, dun2023efficient, hu2023federated,  dun2022resist} or RaPTr \citep{panigrahi2024efficient} and acts as a form of aggressive, structured dropout, providing strong regularization properties. By repeatedly forming new subnets throughout training using a random heuristic, \TwIST{} unlocks a tradeoff between the full chaos of random pruning without further finetuning and the stability of a prune-at-initialization (PaI) approach. Conceptually, the dynamic resembles block coordinate descent, as the system iteratively optimizes distinct groups of parameters (i.e., subnets) \citep{beck2013convergence}, and the entire system dynamics can be understood through the lens of an exploration-exploitation trade-off \citep{gupta2006interplay}.

\textit{Exploration.} A small $\mathcal{C}$, $\tfrac{N_{\text{sub}}}{N_{\text{full}}}$, or repartition interval, each promote exploration as the system continuously trains samples from a diverse population of random subnets. This injection of stochasticity, compared to a static PaI method, helps prevent the optimization from getting caught in sharp local minima \citep{evci2019difficulty, frankle2020pruning, kumar2024no}. Moreover, high physical subnet diversity facilitates the creation of more independent, functionally diverse subnets, giving all subnets an equal opportunity to train and effectively creating a population of ``lottery tickets'' \citep{evci2020rigging}. As seen with ensemble learning, functional diversity is a well-established method for improving robustness to perturbations and adversarial attacks, as the ensemble members are less likely to share common failure modes \citep{pang2019improving, fort2019deep}.

\textit{Exploitation.} Conversely, a large $\mathcal{C}$, $\tfrac{N_{\text{sub}}}{N_{\text{full}}}$, or repartition interval each promote exploitation as the pool of available subnets becomes increasingly static, and the objective narrows from exploring a variety subnets to primarily training a fixed set of shared weights as in a PaI method \citep{Lee2018SNIP, Wang2020GraSP}. This strategy yields two primary, interconnected benefits. First, it fosters \textit{greater network alignment}; by forcing all subnets to co-train the same set of shared core parameters, it encourages them to find common, functionally similar solutions. This is particularly effective as $\mathcal{C}$ is intended to capture the components most critical to performance \citep{michel2019sixteenheadsreallybetter} and is analogous to hard parameter sharing in multi-task learning \citep{Caruana1997MTL}. Second, this alignment, in turn, \textit{stabilizes central model performance}, as the optimization process converges more consistently by exploiting a known set of functionalities.

\textit{The setup for success.} This tradeoff reveals \TwIST's hyperparameters can be tuned to prioritize competing objectives. Leaning into exploration (via small $\mathcal{C}$, low $\tfrac{N_{\text{sub}}}{N_{\text{full}}}$, or frequent repartitioning) promotes functional diversity and robustness. In contrast, leaning into exploitation (via large $\mathcal{C}$, high $\tfrac{N_{\text{sub}}}{N_{\text{full}}}$, or infrequent repartitioning) fosters network alignment for a more stable and rapidly converging central model. As our work focuses on training for pruning, we choose hyperparameters (See Section~\ref{sec:experiments}) that lean into exploration. This choice provides the necessary stochasticity to avoid poor local minima and form a population of robust subnets, while still exploiting the shared structure $\mathcal{C}$ to maintain training stability and central model performance.

\section{Implementation}
\label{sec:implementation}

This section presents \TwIST's implementation. We begin with a careful treatment of the three \TwIST{} variants that differ in their fidelity to physical partitioning. We then break down the pruning techniques in the \TwIST{} framework, and provide the hardware configurations used in our experiments. 

\subsection{\TwIST{} Variants}\label{sec:twist_variants}
Our \TwIST{} variants include Masked \TwIST, True \TwIST, and Hybrid \TwIST. Each variant reflects a different level of faithfulness to physically partitioning parameters and requires a different degree of modification to the source model. Masked \TwIST{} is the simplest to implement, while both True \TwIST{} and Hybrid \TwIST{} require direct modification of the Hugging Face source code. We describe the three variants in order of increasing implementation complexity.

\textbf{Masked \TwIST.} Masked \TwIST{} is a simulated version of \TwIST. Instead of physically scattering the parameters across workers, we emulate subnet training by masking out activations that correspond to inactive blocks.  

For the feedforward layers, we implement this by defining a \texttt{MaskingHook}, which zeros out activations specified by a mask matrix. For the attention layers, we make use of the existing \texttt{head\_mask} argument in Hugging Face's implementation. In both cases, the module output is multiplied element-wise by the mask $\bm{M}$, producing:
\begin{equation*}
\bm{y} = \bm{M} \odot f(\bm{x}),
\end{equation*}
so that only the active subnet contributes to the output.  
During evaluation, we disable masking by setting every element of $\bm{M}$ to 1.

This method does not require any source code modification, making it the fastest and most portable way to prototype \TwIST's behavior while maintaining functional equivalence to True \TwIST{} at the level of gradient flow and scaling.

\textbf{True \TwIST.} True \TwIST{} introduces real physical partitioning of the model parameters and often requires changes to the transformer source code.  
In Hugging Face's original GPT-2 implementation, each attention layer defines square projection matrices with:
$\bm{W}^Q, \bm{W}^K, \bm{W}^V \in \mathbb{R}^{d_{\text{model}} \times d_{\text{model}}}$.
To support physically smaller subnets, we must untie the width of an attention weight matrix from its height. For example, $\bm{W}^Q$ is the concatentation of weights for several heads and thus the second dimension of $\bm{W}^Q$ varies with the number of heads while the first remains fixed as it represents the token embedding dimension. Additionally, we define each layer's width independently of other layers to add capability for sharing some layers while splitting others. 

At runtime, a \textit{central hardware accelerator} maintains the complete model parameters, while $S-1$ workers each hold a smaller physical subset of those parameters. The central accelerator handles parameter partitioning, scattering, and aggregation, whereas each worker locally trains its assigned subnet. This design yields memory and communication savings because each worker stores, updates, and sends only the subset of parameters allocated to its subnet.

\textbf{Hybrid \TwIST.} Hybrid \TwIST{} combines the physical subnets of True \TwIST{} with the masking strategy of Masked \TwIST. The central accelerator participates in training as a regular worker while maintaining the complete parameter set. To prevent duplication of model copies on the central device, we apply masking to the central model so that it behaves like a smaller subnet during training. The remaining $S-1$ workers train physically smaller subnets, constructed and synchronized in the same way as in True \TwIST.  


Figure~\ref{fig:TwIST_variant} summarizes the three \TwIST{} variants.

\begin{figure}[htbp]
    \centering
    \includegraphics[width=0.95\columnwidth]{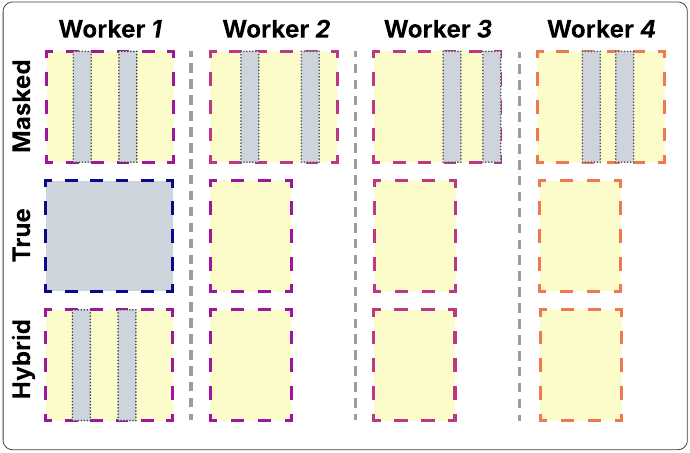}
    \caption{Visualization of the three \TwIST{} variants and their training dynamics in the case of $S = 4$, where Worker $1$ is the central accelerator. In Masked \TwIST, subnets are simulated by masking activations within a single shared model. In True \TwIST, each worker trains a physically smaller subnet that is scattered from and later synchronized with the central model. In Hybrid \TwIST, the central model participates in training as a masked subnet while the remaining workers train physical subnets. Within each repartition interval, the yellow regions indicate active parameters being updated during training, while the grey regions denote inactive parameters that are frozen or masked out.}
    \label{fig:TwIST_variant}
\end{figure}

\subsection{Pruning Implementation}\label{sec:pruning_modules}

The pruning pipeline in \TwIST{} focuses on the same transformer modules that are partitioned during subnetwork training. To maintain flexibility, pruning is implemented through a modular architecture composed of several interoperable components. We describe the key components below and summarize how they are integrated into full multi-stage pruning workflows in the experimental section.

\textbf{Deployment.} To generate the pruned model for deployment, we mirror the subnet generation procedure used during \TwIST{} training. These randomly sampled subnets allow cheap and efficient deployment of pruned models without requiring any complex search procedures or retraining, forming the foundation of \TwIST{} by allowing the deployment of models of varying sizes from a single checkpoint.

\textbf{Training Backends.} We provide two distinct backends for training models. The \texttt{DDP} module trains the given model under data parallelism, where all workers share identical parameters. The \texttt{twist} module trains the model with the Masked \TwIST{} methodology in Section~\ref{sec:twist_variants}. 

\textbf{Evaluation Mode.} Our evaluation implementation varies depending on whether we want to evaluate the subnetwork or the full model. To evaluate the subnetwork, on one hand, we mask parameters as needed and apply the appropriate scaling hooks before performing evaluation on each worker’s subnet. Each subnet is evaluated independently, and the resulting test losses are averaged across workers. To evaluate the full model, on the other hand, we assess the model directly without additional scaling or masking.

\textbf{Hardware Setup.} We set deterministic seeds to ensure reproducibility across our experiments. Our experiments used four Tesla P100 GPUs, each with 16 GB of memory, connected via Peripheral Component Interconnect express (PCIe) for synchronized parallel execution.

\section{Experiments}
\label{sec:experiments}

\begin{table*}[htbp]
\small
\centering
\begin{tabular}{lcccccccc}
\toprule
\multirow{2}{*}{\textbf{Setting}} & \multirow{2}{*}{\textbf{Ratio ($\bm{\kappa}$)}} &
\multicolumn{5}{c}{\textbf{Data Parallelism}} &
\multicolumn{2}{c}{\textbf{TwIST} (Ours)} \\
\cmidrule(lr){3-7}\cmidrule(lr){8-9}
& & SparseGPT & Wanda & Z-Pruner & Block Prune & $SE$ & $SE$ & $SE_{6/12}$ \\
\midrule

\multirow{4}{*}{\texttt{attn}}
  &  8/12 & \textbf{16.14} & 16.42 & 16.27 & 21.90 & 22.91 & 17.35 & 17.29 \\
  &  6/12 & \textbf{16.81} & 19.33 & 17.75 & 30.20 & 32.56 & 17.74 & 17.74 \\
  &  4/12 & 20.06 & 86.09 & 27.42 & 78.26 & 61.12 & \textbf{18.25} & 18.88 \\
  &  3/12 & 25.76 & 57.72 & 44.12 & 133.41 & 273.46 & \textbf{18.49} & 19.90 \\
\midrule

\multirow{4}{*}{\texttt{both}}
  &  8/12 & \textbf{16.60} & 18.30 & 17.27 & 66.44 & 54.19 & 19.67 & 20.02 \\
  &  6/12 & \textbf{19.11} & 36.19 & 24.85 & 247.01 & 176.71 & 21.32 & 21.32 \\
  &  4/12 & 31.64 & 173.66 & 109.11 & 617.98 & 974.31 & \textbf{23.14} & 25.53 \\
  &  3/12 & 58.16 & 284.59 & 238.09 & 1156.67 & 2587.42 & \textbf{24.86} & 24.32 \\
\midrule

\multirow{4}{*}{\texttt{ffn}}
  &  8/12 & \textbf{16.41} & 17.55 & 16.85 & 33.83 & 31.60 & 18.06 & 18.48 \\
  &  6/12 & \textbf{17.99} & 24.57 & 20.65 & 62.93 & 63.23 & 18.90 & 18.90 \\
  &  4/12 & 24.00 & 63.14 & 41.57 & 162.74 & 174.95 & \textbf{19.69} & 21.39 \\
  &  3/12 & 32.63 & 146.65 & 88.64 & 298.84 & 361.33 & \textbf{20.51} & 23.51 \\
\bottomrule
\end{tabular}
\caption{Performance of subnets in terms of Perplexity (PPL).}
\label{tab:subnet_performance}
\end{table*}

We construct our experiments to verify: a) \TwIST{} allows us to identify effective subnets at varying sparsity levels without any post-training pruning (Section~\ref{sec:twist_subnet_quality}); b) \TwIST{} shows less communication cost and speeds up the training process compared to the standard data-parallel method (Section~\ref{sec:twist_training_performance}).

\textbf{Setup.} Our main experiments use the decoder-only GPT-2 causal language model with 124M parameters \citep{radford2019language}. Within each experiment, we initialize the model with GPT-2's original configuration (12 transformer layers, 12 attention heads per layer, and a default hidden size of 768) and finetune its pretrained checkpoint for 3 epochs on the next token prediction task. We train the model using the Adam optimizer with a learning rate of $1 \times 10^{-4}$. 

For $\TwIST{}$, we set $|C| = 0$ for all layers; and share the first two and last two layers (i.e. $N_\text{sub} = N_\text{full} = 12$). The remaining 10 layers are partitioned and have the same number of blocks across all subnetworks. For these 10 layers, we vary $N_\text{sub}$ across experiments. For each configuration, we evaluate three subsampling settings: attention-only (\texttt{attn}), feedforward-only (\texttt{ffn}), and combined (\texttt{both}), where subsampling is applied exclusively to the respective area. We fix the repartition interval at 15 training batches.

For the training dataset, we use the official training split of WikiText-103-raw-v1, a standard benchmark for language modeling \citep{merity2016pointer}. By default, we use the GPT-2 BPE tokenizer to preprocess the dataset, ensure each training sequence consists of 1024 consecutive tokens, and set the batch size to 2.

\textbf{Baseline Pipelines and Evaluation Metrics.} We apply \TwIST{} in a \textbf{two-stage pipeline} where we first train the model by \TwIST{} and then extract a random subnet from the model using the module introduced in Section~\ref{sec:pruning_modules}. For our two-stage pipeline baselines, we first train the model by the standard data-parallel method \citep{li2020pytorchdistributedexperiencesaccelerating} and then obtain a subnet either by random subnet extraction or by the following post-training pruning methods:
\begin{itemize}[leftmargin=*]
  \item \textbf{SparseGPT} is a one-shot pruning approach that zeros a chosen subset of weights, then updates the kept weights by solving a local quadratic (Hessian-weighted) reconstruction problem so the layer’s responses on the calibration data match the original responses as closely as possible \citep{frantar2023sparsegpt}.
  \item \textbf{Wanda} is a one-shot pruning approach that ranks weights by combining their magnitudes with input activation statistics. Pruning decisions are made using a small calibration set as well\citep{sun2024simpleeffectivepruningapproach}.
  \item \textbf{Z-Pruner} is a zero-shot structured pruning technique that does not rely on calibration data. It estimates weight importance using analytic, data-free proxies (e.g., norm-based sensitivity) and removes less important channels or heads in a single step \citep{bhuiyan2025zprunerposttrainingpruninglarge}.
  \item \textbf{Block Prune} is a one-shot pruning method that prunes on weight matrices corresponding to attention and feedforward projections by partitioning each matrix into parameter blocks, and estimating the empirical Fisher of each block using a full pass over the dataset. In which the block with the lowest Fisher information are masked out \citep{michel2019sixteenheadsreallybetter}.
\end{itemize}

\subsection{Subnet Quality} 
\label{sec:twist_subnet_quality}

We evaluate subnet quality using evaluation perplexity on the WikiText-103 test split and report the results in Table~\ref{tab:subnet_performance}. For subnet extraction ($SE$), each attention and feedforward layer is divided into 12 parameter blocks, following the same block partitioning procedure described in Section~\ref{sec:subnet_generator}. To ensure a fair comparison, all pipelines prune the same parameter types (\texttt{attn}, \texttt{ffn}, or \texttt{both}) to identical sparsity levels. We define the subnet ratio, $\kappa$, by the ratio $X/12$, where $X$ indicates the number of the remaining parameter blocks after pruning or $SE$. A smaller ratio corresponds to higher sparsity (i.e., fewer active parameters in the network). When the training and extraction sparsity differ, we denote the non-default setting as $SE_{X/12}$, where $X/12$ represents the sparsity ratio used during training. For example, $SE_{6/12}$ indicates that the \TwIST{} model was trained with a sparsity level of $6/12$, while retaining the capability to deploy subnets at different sparsity levels.


\TwIST{} achieves perplexity scores remarkably close to State-Of-The-Art (SOTA) one-shot methods like SparseGPT \citep{frantar2023sparsegpt}, but at effectively \textit{zero post-training cost}. The SOTA PTP baselines incur substantial overhead, requiring calibration data and complex, layer-wise reconstructions involving costly inverse Hessian ($\bm{H}^{-1}$) computations or approximations (e.g., $O(d_{\text{model}}^3)$) \citep{frantar2023sparsegpt}. In sharp contrast, \TwIST{} bypasses this entirely, using a simple random sampling strategy with negligible overhead, yet achieves highly competitive results.

Furthermore, \TwIST{}'s advantage emerges strongly under aggressive pruning ($\kappa \le 6/12$), where it consistently and significantly outperforms all baselines. For instance, at $\kappa = 4/12$ (\texttt{both}), \TwIST{} (23.14 PPL) drastically outperforms SparseGPT (31.64 PPL), while other methods collapse (PPL > 100). Critically, unlike the PTP baselines, \TwIST{} is a \textbf{structured pruning} method that removes entire blocks. This creates genuinely smaller, dense matrices, enabling tangible inference speedups and memory savings on commodity hardware. Thus, \TwIST{} delivers high-sparsity subnets that are not only accurate but also genuinely faster in real-world deployments.

\begin{figure}[htbp]                                      
\centering
\includegraphics[width=0.95\columnwidth]{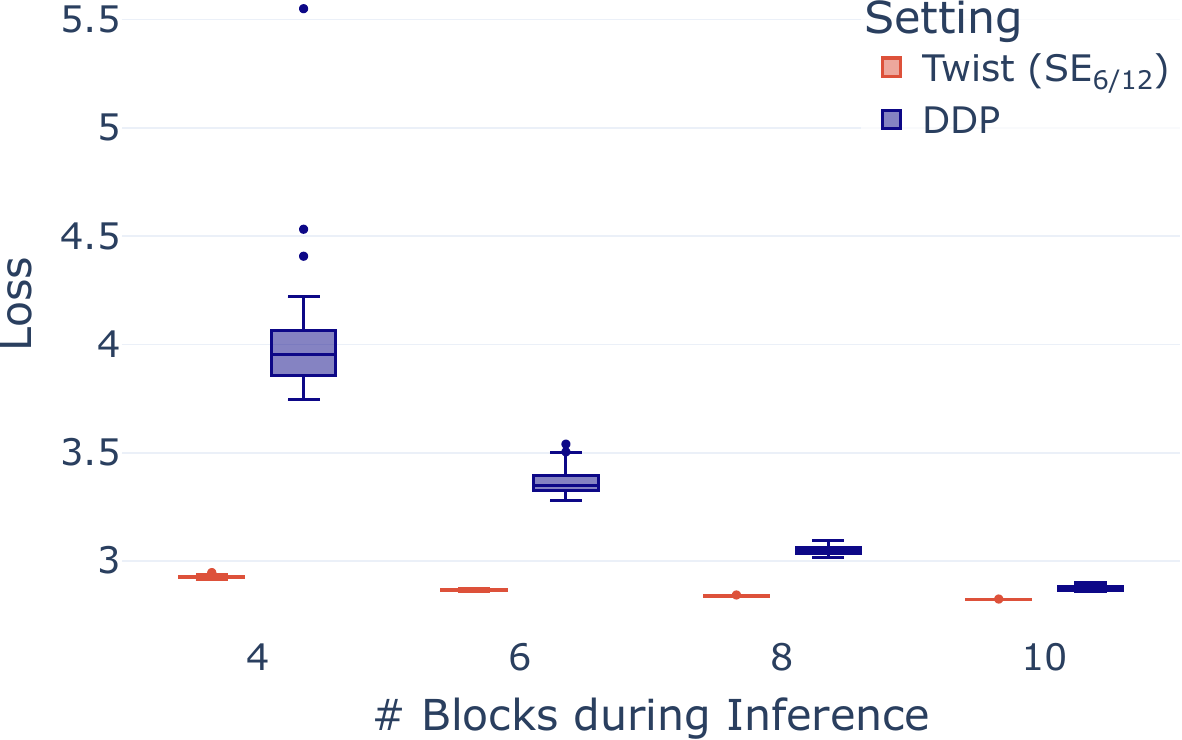} 
\caption{Distribution of eval loss for randomly generated subnets in the \texttt{attn} configuration. The distributions for \TwIST{} ($SE_{6/12}$) are compared against a \texttt{DDP} baseline across various subnet ratios. The $SE_{6/12}$ variant of \TwIST{} is presented for a direct comparison, as both this method and \texttt{DDP} involve only a single training pass.}
\label{fig:subnet_loss_distribution_attn}
\end{figure}

\textbf{System Stability.} Figure~\ref{fig:subnet_loss_distribution_attn} plots the eval loss distributions for subnets identified by Twist ($SE_{6/12}$) and the \texttt{DDP} baseline. These results, shown for the \texttt{attn} setting, demonstrate that the \TwIST{} distributions are sharply concentrated, indicating low variance across different random subnet generations. In contrast, \texttt{DDP} exhibits significantly higher variance. This variance in \texttt{DDP}-generated subnets becomes more pronounced as the level of sparsity increases.

These findings suggest that the performance of \texttt{DDP} subnets is sensitive to the specific subnet initialization, a characteristic consistent with the LTH. Conversely, the low variance of \TwIST{} suggests that its method identifies subnets that are more structurally consistent and less dependent on the random generation process. This stability is crucial: it indicates that the performance of any single sampled subnet is a strong proxy for the performance of other subnets, making the results from \TwIST{} highly reliable and reproducible. We note that similar distributional trends were observed in the \texttt{ffn} and \texttt{both} settings (See Appendix~\ref{sec:additional_system_stability}).

\begin{figure}[hbp]
\centering
\includegraphics[width=0.75\columnwidth]{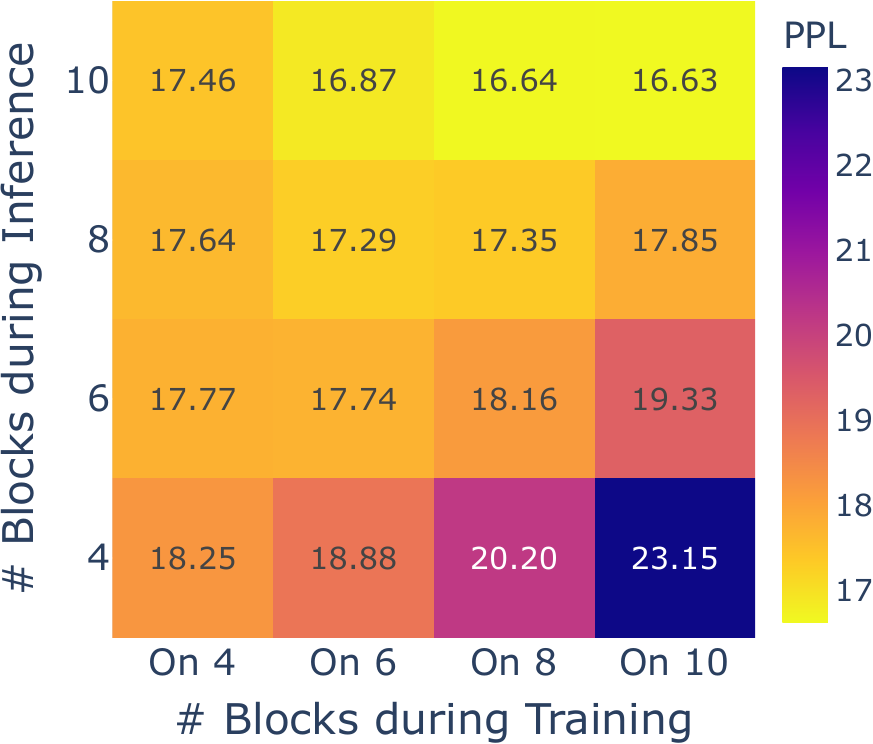} 
\caption{Heatmap of subnet robustness for the \texttt{attn} setting. Brighter colors (yellow) signify lower PPL (better performance), and darker colors (blue) signify higher PPL.}
\label{fig:robustness_heatmap_attn}
\end{figure}
\textbf{Architectural Robustness.} Next, we investigate architectural robustness, the model's ability to generalize to subnet configurations it was not trained for. Figure~\ref{fig:robustness_heatmap_attn} presents a heatmap of subnet performance across various mismatched training and evaluation sparsity targets for the \texttt{attn} setting. See Appendix~\ref{sec:additional_architectural_robustness} for the \texttt{both} and \texttt{ffn} settings.

In Figure~\ref{fig:robustness_heatmap_attn}, the horizontal coordinate ($X$) represents the number of parameter blocks (out of 12) used per layer during the training phase of \TwIST{}. The vertical coordinate ($Y$) represents the number of blocks used for evaluation. A mismatch occurs when $Y \neq X$, with the diagonal representing the matched baseline (where $Y = X$).

The heatmaps show similar trends across pruned parameter types. As expected, performance is generally optimal along the diagonal ($Y = X$). Yet, two key asymmetries emerge. First, ``inferencing downward'' (training on a large subnet and evaluating on a smaller one, $X > Y$) causes a sharp performance degradation. This is most severe when the mismatch is large (e.g., $X = 10$, $Y = 4$), where PPL can increase by $\approx 42$ points (from 18 to 60). Second, ``inferencing upward'' (training on a small subnet and evaluating on a larger one, $X < Y$) often maintains or even improves performance compared to the matched baseline. In other words, we find that it is more difficult to prune (go smaller) than to expand (go bigger) at inference time. In practice, this implies one should be more conservative when selecting a training size for an unknown target device.

Most interestingly, training at a mid-level sparsity (e.g., $X = 6$) yields the most uniformly low perplexity across all evaluation sizes ($Y$). This identifies a ``sweet spot'' for training: if the target deployment sparsity is unknown, training \TwIST{} at a mid-level sparsity provides a highly robust parent model that minimizes worst-case degradation, transferring well to both leaner and denser subnet configurations.

The practical benefits of this robustness are significant: a single \TwIST{}-trained parent model can deploy a spectrum of ``tickets'' tailored to the computational budget of each device, enabling a flexible and efficient distribution paradigm that moves away from a one-size-fits-all deployment. When combined with the system stability in Figure~\ref{fig:subnet_loss_distribution_attn}, \TwIST{} makes models resilient to hardware failures in distributed, model-parallel deployments \citep{shoeybi2019megatronlm, lepikhin2020gshard}. In such large-scale settings, the failure of a node (which holds certain parameters) is analogous to ``inferencing downward'' to a smaller subnet. Our results suggest models trained with \TwIST{} can gracefully handle such failures, a critical feature for large-scale training and inference \citep{pytorch2024faulttolerant, nebius2025faulttolerant}.

\subsection{Training Efficiency} 
\label{sec:twist_training_performance}


\begin{figure}[htbp]
\centering
\includegraphics[width=0.95\columnwidth]{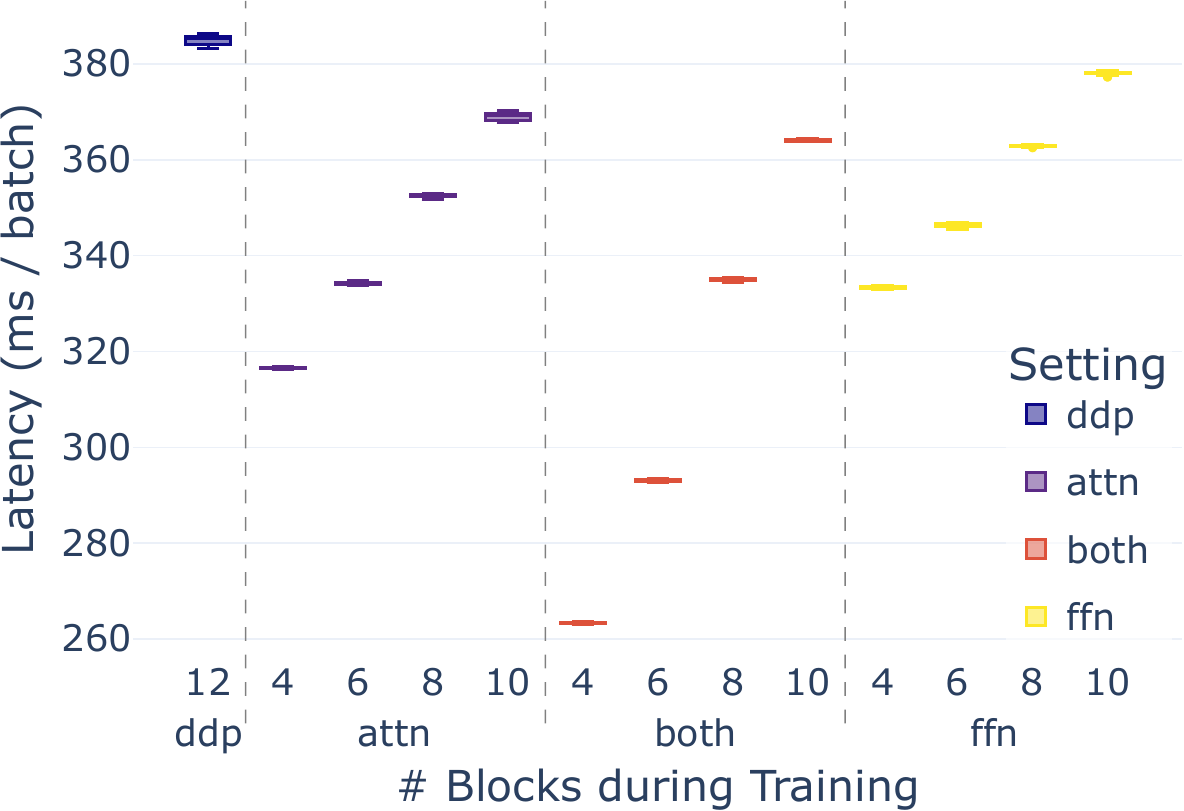}
\caption{Training latency ablation.}
\label{fig:latency_ablation}
\end{figure}

\begin{table}[htbp]
\small
\centering
\begin{tabular}{l r c c c c c}
\toprule
\textbf{Setting} & $\bm{\kappa}$ & \textbf{memory} & \textbf{comm} & \textbf{wall clock} \\
\midrule

\multirow{4}{*}{\texttt{attn}}
 & 10/12 & 1.87 & 44.35 & 4:19:06 \\
 &  8/12 & 1.82 & 43.20 & 4:08:30 \\
 &  6/12 & 1.79 & 42.05 & 3:52:27 \\
 &  4/12 & 1.72 & 40.90 & 3:41:12 \\
\midrule

\multirow{4}{*}{\texttt{both}}
 & 10/12 & 1.77 & 42.05 & 4:14:03 \\
 &  8/12 & 1.64 & 38.60 & 3:49:50 \\
 &  6/12 & 1.5 & 35.14 & 3:25:14 \\
 &  4/12 & 1.34 & 31.69 & 3:04:39 \\
\midrule

\multirow{4}{*}{\texttt{ffn}}
 & 10/12 & 1.82 & 43.20 & 4:30:31 \\
 &  8/12 & 1.73 & 40.90 & 4:14:19 \\
 &  6/12 & 1.63 & 38.60 & 4:06:53 \\
 &  4/12 & 1.53 & 36.29 & 3:54:19 \\

\midrule

\multirow{1}{*}{\texttt{DDP}}
& 12/12 & 1.88 & 45.50 & 4:35:42 \\

\bottomrule
\end{tabular}
\caption{Tabulation of memory (GB), communication (TB), and wall clock (HH:MM) across pruning settings.}
\label{tab:ablation_metrics}
\end{table}
Theoretically, \TwIST{} reduces the number of model parameters, which is expected to decrease memory usage and communication costs (see Appendix~\ref{sec:memory_and_communication} for a detailed derivation). We empirically validate the resulting performance gain by measuring training latency. Figure~\ref{fig:latency_ablation} plots the latency, measured in milliseconds per batch (ms / batch), where lower values indicate faster performance. We compare \TwIST{} to the Distributed Data Parallel (\texttt{DDP}) method, representing full-model fine-tuning and serving as our baseline. \texttt{DDP} exhibits the slowest speed at approximately 385 ms/batch.

As expected, we observe throughput increases as the number of blocks in a subnet decreases from ten to four: splitting only attention blocks (\texttt{attn}), splitting only feedforward network blocks (\texttt{ffn}), and splitting (\texttt{both}). Compared to the \texttt{DDP} baseline, four block subnet strategies offer significant efficiency gains. The \texttt{ffn} method ($\approx$334 ms/batch) achieves a 1.15 speedup (a 13.2\% time reduction), while the \texttt{attn} method ($\approx$317 ms/batch) reaches a 1.21 speedup (a 17.7\% time reduction). Our proposed \texttt{both} approach ($\approx$263 ms/batch) is markedly superior, delivering a 1.46 speedup and cutting training time by 31.7\% relative to the full-model baseline, marking a computational advantage.

Table~\ref{tab:ablation_metrics} reports memory (GB), communication volume (TB), and wall clock time across \TwIST{} with subnet ratio $\kappa$ and the \texttt{DDP} baseline on four compute nodes. Memory is the amount of GPU memory allocated on each compute node. Communication is the total inter-node training traffic, defined in Appendix~\ref{sec:memory_and_communication}. Wall clock time is the training duration, including evaluation, for three epochs.

Similar to the trend we observe in throughput, all three metrics improve monotonically as $\kappa$ decreases. The \texttt{both} setting at $\kappa = 4 / 12$ yields the largest gains relative to \texttt{DDP}, with memory save of 0.54 GB (28.7\%); communication decrease of 13.81 TB (30.4\%); and wall clock time reduction of about 1.5 hours (33.0\%), corresponding to a $1.49 \times$ speedup. Interestingly, although \texttt{attn} communicates more than \texttt{ffn} at the same $\kappa$ (for $\kappa = 4 / 12$: 40.90 TB versus 36.29 TB), it still finishes faster with a 13 minutes advantage. This discrepancy likely arises because attention blocks are computationally heavier than feedforward blocks, so removing them reduces a greater portion of the overall computational per subnetwork. 

\section{Conclusion and Future Work}
\label{sec:conclusion_future}

We introduced \TwIST{}, a distributed system where subnets are trained in parallel and periodically aggregated. This method is motivated by the \textit{golden lottery ticket hypothesis} and validates that randomly sampled subnets from a \TwIST{}-trained network can achieve high performance \textit{without} fine-tuning. As shown in Table~\ref{tab:subnet_performance}, this enables robust, zero-cost pruning at deployment, achieving perplexity scores competitive with SOTA methods. \TwIST{} effectively shifts the search for sparse models from an expensive post-training step (which often requires calibration data and complex operations like inverting the Hessian matrix) to the training process itself.

\TwIST's true advantage emerges under aggressive pruning ($\kappa \le 6/12$), where it consistently outperforms all baselines (e.g., 23.14 PPL vs. 31.64 for SparseGPT at $\kappa = 4/12$). Critically, as a \textit{structured pruning} method, \TwIST{} removes entire parameter blocks. This design produces genuinely smaller, dense matrices that translate directly to tangible inference speedups and memory savings on commodity hardware (e.g., CPUs) that lack efficient sparse computation support, offering a practical path to high-sparsity models.

This study was limited by budget, precluding validation on multibillion-parameter models. Future work should explore more sophisticated subnet assignment strategies, optimal aggregation scheduling, and the use of dynamic architectures.

\clearpage

\bibliography{refs}
\bibliographystyle{mlsys2025}

\clearpage
\appendix

\section{Background and Notation} \label{sec:notation}
This section introduces the fundamental components and notation for the transformer architecture, primarily following the conventions from the paper by \citet{panigrahi2024efficient}. 

\subsection{Key Components}

\textbf{Layer Normalization.} As defined in \citet{panigrahi2024efficient}, a layer normalization function $f_{\text{ln}} : \mathbb{R}^{d_{\text{model}}} \to \mathbb{R}^{d_{\text{model}}}$ with learnable parameters $\bm{\gamma}, \bm{b} \in \mathbb{R}^{d_{\text{model}}}$ operates on an input vector $\bm{x} \in \mathbb{R}^{d_{\text{model}}}$. The process first normalizes $\bm{x}$ to $\bm{z} = (\bm{x}-\mu)/\sigma$, where $\mu$ and $\sigma$ are the mean and standard deviation of the elements in $\bm{x}$. The final output $\bm{y}_{\text{ln}}$ is then computed as $\bm{y}_{\text{ln}} = \bm{\gamma} \odot \bm{z} + \bm{b}$.

\textbf{Multi-Head Attention.} From \citet{vaswani2017attention}, the scaled dot-product attention function for a single head is defined as
\begin{equation*}
\text{Attention}(\bm{Q}, \bm{K}, \bm{V}) = \text{softmax}\left(\frac{\bm{Q}\bm{K}^T}{\sqrt{d_{\text{head}}}}\right) \bm{V},
\end{equation*}
where $\bm{Q} \in \mathbb{R}^{N \times d_{\text{head}}}$, $\bm{K} \in \mathbb{R}^{N \times d_{\text{head}}}$, and $\bm{V} \in \mathbb{R}^{N \times d_{\text{head}}}$. The output of a single head is a matrix of weighted sums of the value vectors, where the weights are determined by the dot-product similarity between the queries and keys.

Multi-head attention allows a model to jointly attend to information from various representation subspaces at different positions. This is a significant improvement over single-head attention. A multi-head attention layer, denoted as $f_{\text{attn}}$, processes a sequence of input vectors $\{\bm{x}_n\}_{n=1}^N$ and outputs a sequence $\{\bm{y}_n\}_{n=1}^N$. The layer uses $H$ attention heads, where each head $h$ has its own parameter matrices $\{\bm{W}^Q_h, \bm{W}^K_h, \bm{W}^V_h \in \mathbb{R}^{d_{\text{model}} \times d_{\text{head}}}\}$.

The output is the concatenation of the outputs from each head, projected by a final matrix $\bm{C}^{\text{attn}} \in \mathbb{R}^{H d_{\text{head}} \times d_{\text{model}}}$:
\begin{equation*}
\bm{Y}_{\text{attn}} = \text{Concat}(\text{head}_1, \dots, \text{head}_H) \bm{C}^{\text{attn}},
\end{equation*}
where $\text{head}_h = \text{Attention}(\bm{Q}_h, \bm{K}_h, \bm{V}_h)$. The query, key, and value matrices for each head are derived from the input as $\bm{Q}_h = \bm{X} \bm{W}^Q_h$, $\bm{K}_h = \bm{X} \bm{W}^K_h$, and $\bm{V}_h = \bm{X} \bm{W}^V_h$, with $\bm{X} \in \mathbb{R}^{N \times d_{\text{model}}}$ being the matrix of input vectors.

\textbf{Feedforward.} A feedforward network (FFN) layer, $f_{\text{ffn}} : \mathbb{R}^{d_{\text{model}}} \to \mathbb{R}^{d_{\text{model}}}$, is defined with parameters $\{\bm{W}^{\text{ffn}} \in \mathbb{R}^{d_{\text{inner}} \times d_{\text{model}}}, \bm{C}^{\text{ffn}} \in \mathbb{R}^{d_{\text{model}} \times d_{\text{inner}}}\}$ and uses the $\sigma_{\text{relu}}$ activation function. For an input vector $\bm{x} \in \mathbb{R}^{d_{\text{model}}}$, the output $\bm{y} \in \mathbb{R}^{d_{\text{model}}}$ is given by
\begin{equation*}
\bm{y}_{\text{ffn}} = \bm{C}^{\text{ffn}}\sigma_{\text{relu}}(\bm{W}^{\text{ffn}} \bm{x}).
\end{equation*}

\subsection{Transformer Layer Architecture}
A pre-layernorm transformer layer integrates the above components in a sequential manner as described in \citet{panigrahi2024efficient}. Given an input sequence matrix $\bm{X} \in \mathbb{R}^{N \times d_{\text{model}}}$, the layer produces an output matrix $\bm{Y} \in \mathbb{R}^{N \times d_{\text{model}}}$ through the following steps:
\begin{enumerate}
    \item \textbf{Attention Layer Normalization}: The input vectors are first normalized. Let $\bm{X}_{\text{attnln}}$ be the matrix where each row is the result of applying layer normalization to the corresponding row of $\bm{X}$: $\bm{X}_{\text{attnln}} = f_{\text{ln}}(\bm{X}; \bm{\gamma}_{\text{attn}}, \bm{b}_{\text{attn}})$.
    \item \textbf{Multi-Head Attention}: The normalized vectors are passed through the multi-head attention layer: $\bm{Y}_{\text{attn}} = f_{\text{attn}}(\bm{X}_{\text{attnln}})$.
    \item \textbf{Residual Connection}: A residual connection is added: $\bm{Y}_{\text{attnlayer}} = \bm{X} + \bm{Y}_{\text{attn}}$.
    \item \textbf{FFN Layer Normalization}: The result is normalized before the FFN: $\bm{X}_{\text{ffnln}} = f_{\text{ln}}(\bm{Y}_{\text{attnlayer}}; \bm{\gamma}_{\text{ffn}}, \bm{b}_{\text{ffn}})$.
    \item \textbf{FFN Function}: The normalized vectors are processed by the FFN. Since the FFN operates on individual vectors, this is applied row-wise: $\bm{Y}_{\text{ffn}} = f_{\text{ffn}}(\bm{X}_{\text{ffnln}})$.
    \item \textbf{Final Output}: A final residual connection yields the output of the layer: $\bm{Y} = \bm{Y}_{\text{attnlayer}} + \bm{Y}_{\text{ffn}}$.
\end{enumerate}

\subsection{Initialization}
The weights for the transformer layer are initialized following \citet{he2015delving}:
\begin{align*}
\bm{W}^Q_h, \bm{W}^K_h, \bm{W}^V_h &\sim \mathcal{N}\left(0, \frac{1}{d_{\text{model}}} \bm{I}\right) \\
\bm{C}^{\text{attn}} &\sim \mathcal{N}\left(0, \frac{1}{H d_{\text{head}}} \bm{I}\right) \\
\bm{W}^{\text{ffn}} &\sim \mathcal{N}\left(0, \frac{2}{d_{\text{model}}} \bm{I}\right) \\
\bm{C}^{\text{ffn}} &\sim \mathcal{N}\left(0, \frac{1}{d_{\text{inner}}} \bm{I}\right)
\end{align*}
The layer normalization parameters $\bm{\gamma}$ and $\bm{b}$ for both $f_{\text{LN}}^{\text{attn}}$ and $f_{\text{LN}}^{\text{ffn}}$ are initialized to 1 and 0, respectively.

For much of the proofs below, we follow in He's convention \citep{he2015delving} and assume $\bm{X}$ has i.i.d. components.

\begin{lemma}
\label{lma:pre_activation}
Let $\bm{x} \in \mathbb{R}^{d_{\text{in}}}$ be a random vector with i.i.d. components. Let $\bm{W} \in \mathbb{R}^{d_{\text{out}} \times d_{\text{in}}}$ be a random matrix with i.i.d. components, $\bm{W}_{ji} \sim \mathcal{N}(0, \sigma_W^2)$. If we let $\bm{y} := \bm{W} \bm{x}$, then the expected squared norm of the output is
\begin{equation*}
\mathbb{E}[\norm{\bm{y}}^2] = d_{\text{out}} \sigma_W^2 \mathbb{E}[\norm{\bm{x}}^2]
\end{equation*}
\end{lemma}
\begin{proof}
We expand the squared norm and use the linearity of expectation:
\begin{align*}
\mathbb{E}[\norm{\bm{y}}^2] &= \mathbb{E} \left[ \sum_{j=1}^{d_{\text{out}}} \left(\sum_{i=1}^{d_{\text{in}}} \bm{W}_{ji} \bm{x}_i\right)^2 \right] \\
&= \sum_{j=1}^{d_{\text{out}}} \mathbb{E} \left[ \sum_{i, k} \bm{W}_{ji} \bm{W}_{jk} \bm{x}_i \bm{x}_k \right].
\end{align*}
Since the components of $\bm{W}$ are i.i.d. with zero mean and are independent of $\bm{x}$, the cross-terms vanish on expectation: $\mathbb{E}[\bm{W}_{ji}\bm{W}_{jk}\bm{x}_i\bm{x}_k] = \mathbb{E}[\bm{W}_{ji}]\mathbb{E}[\bm{W}_{jk}]\mathbb{E}[\bm{x}_i\bm{x}_k] = 0$. We are left only with the terms where $i = k$:
\begin{align*}
\mathbb{E}[\norm{\bm{y}}^2] &= \sum_{j=1}^{d_{\text{out}}} \sum_{i=1}^{d_{\text{in}}} \mathbb{E}[\bm{W}_{ji}^2 \bm{x}_i^2] \\
&= \sum_{j=1}^{d_{\text{out}}} \sum_{i=1}^{d_{\text{in}}} \mathbb{E}[\bm{W}_{ji}^2] \mathbb{E}[\bm{x}_i^2] \\
&= \sum_{j=1}^{d_{\text{out}}} \sum_{i=1}^{d_{\text{in}}} \sigma_W^2 \mathbb{E}[\bm{x}_i^2] \\
&= d_{\text{out}} \sigma_W^2 \sum_{i=1}^{d_{\text{in}}} \mathbb{E}[\bm{x}_i^2] \\
&= d_{\text{out}} \sigma_W^2 \mathbb{E}[\norm{\bm{x}}^2]. \qedhere
\end{align*}
\end{proof}

\begin{lemma}
\label{lma:post_activation}
Let $\bm{x} \in \mathbb{R}^{d_{\text{in}}}$ be a random vector with i.i.d. components. Let $\bm{W} \in \mathbb{R}^{d_{\text{out}} \times d_{\text{in}}}$ be a random matrix with i.i.d. components, $\bm{W}_{ji} \sim \mathcal{N}(0, \sigma_W^2)$. If $\sigma_{\text{relu}}(\cdot)$ is the element-wise ReLU activation function, then the expected squared norm of the output is
\begin{equation*}
\mathbb{E}[\norm{\sigma_{\text{relu}}(\bm{W}\bm{x})}^2] = \frac{1}{2} d_{\text{out}} \sigma_W^2 \mathbb{E}[\norm{\bm{x}}^2]
\end{equation*}
\end{lemma}
\begin{proof}
Let $\bm{y} = \bm{W}\bm{x}$. Since each $\bm{W}_{ji}$ is drawn from a distribution symmetric about 0 and is independent of $\bm{x}$, each pre-activation component $\bm{y}_j = \sum_i \bm{W}_{ji}\bm{x}_i$ also has a distribution symmetric about 0. For any such random variable, this implies $\mathbb{E}[\max(0, \bm{y}_j)^2] = \frac{1}{2}\mathbb{E}[\bm{y}_j^2]$. The expected squared norm is then
\begin{align*}
\mathbb{E}[\norm{\sigma_{\text{relu}}(\bm{y})}^2] &= \sum_{j=1}^{d_{\text{out}}} \mathbb{E}[\max(0, \bm{y}_j)^2] \\
&= \frac{1}{2} \sum_{j=1}^{d_{\text{out}}} \mathbb{E}[\bm{y}_j^2] = \frac{1}{2}\mathbb{E}[\norm{\bm{y}}^2].
\end{align*}
The result follows by substituting $\mathbb{E}[\norm{\bm{y}}^2] = d_{\text{out}} \sigma_W^2 \mathbb{E}[\norm{\bm{x}}^2]$ from Lemma \ref{lma:pre_activation}.
\end{proof}

\begin{theorem}
\label{thm:ffn_squared_scale_factor}
Let $\bm{y}_{\text{ffn}} = \bm{C}^{\text{ffn}}\sigma_{\text{relu}}(\bm{W}^{\text{ffn}} \bm{x})$ be a two-layer FFN with inner dimension $d_{\text{inner}}$ and input $\bm{x} \in \mathbb{R}^{d_{\text{model}}}$ with i.i.d. components. Let $\bm{y}'_{\text{ffn}}$ be the output after reducing the inner dimension to $d'_{\text{inner}} \leq d_{\text{inner}}$ by taking a subset of the original weights $\bm{W}^{\text{ffn}}$ and $\bm{C}^{\text{ffn}}$. The expected squared output norm scales as
\begin{equation*}
\mathbb{E} [\norm{\bm{y}'_{\text{ffn}}}^2] = \frac{d'_{\text{inner}}}{d_{\text{inner}}} \mathbb{E} [\norm{\bm{y}_{\text{ffn}}}^2].
\end{equation*}
\end{theorem}
\begin{proof}
Let $\bm{W}' \in \mathbb{R}^{d'_{\text{inner}} \times d_{\text{model}}}$ and $\bm{C}' \in \mathbb{R}^{d_{\text{model}} \times d'_{\text{inner}}}$ be the submatrices of the original weights. By our initialization, the variances are $\text{Var}(\bm{W}^{\text{ffn}}_{ji}) = \frac{2}{d_{\text{model}}}$ and $\text{Var}(\bm{C}^{\text{ffn}}_{kj}) = \frac{1}{d_{\text{inner}}}$. Since $\bm{W}'$ and $\bm{C}'$ are subsets of these weights, their elements retain the same variances. Now, we derive the expected squared norm of the modified network's output using Lemma \ref{lma:pre_activation} and \ref{lma:post_activation}.
\begin{align*}
\mathbb{E} [\norm{\bm{y}'_{\text{ffn}}}^2]
&= \mathbb{E} [\norm{\bm{C}'\sigma_{\text{relu}}(\bm{W}' \bm{x})}^2] \\
&= d_{\text{model}} \text{Var}(\bm{C}') \mathbb{E} [\norm{\sigma_{\text{relu}}(\bm{W}' \bm{x})}^2] \\
&= d_{\text{model}} \frac{1}{d_{\text{inner}}} \mathbb{E} [\norm{\sigma_{\text{relu}}(\bm{W}' \bm{x})}^2] \\
&= d_{\text{model}} \frac{1}{d_{\text{inner}}} \left( \frac{1}{2} d'_{\text{inner}} \text{Var}(\bm{W}') \mathbb{E} [\norm{\bm{x}}^2] \right) \\
&= d_{\text{model}} \frac{1}{d_{\text{inner}}} \left( \frac{d'_{\text{inner}}}{2} \frac{2}{d_{\text{model}}} \mathbb{E} [\norm{\bm{x}}^2] \right) \\
&= \frac{d'_{\text{inner}}}{d_{\text{inner}}} \mathbb{E} [\norm{\bm{x}}^2].
\end{align*}
For the original network, we can set $d'_{\text{inner}} = d_{\text{inner}}$ in the above derivation, which yields $\mathbb{E} [\norm{\bm{y}_{\text{ffn}}}^2] = \mathbb{E} [\norm{\bm{x}}^2]$. Substituting this back into the previous result gives the theorem.
\end{proof}

\begin{lemma}
\label{lma:ffn_subset_out}
Let $\bm{x} \in \mathbb{R}^{d_{\text{model}}}$ be a random vector with i.i.d. components. Let $\bm{y}_{\text{ffn}}$ and $\bm{y}'_{\text{ffn}}$ be as in Thm.~\ref{thm:ffn_squared_scale_factor}. Then for each $i \in \{1, \dots, d_{\text{model}} \}$,
\begin{equation*}
\text{Var}(\bm{y}'_{\text{ffn}_i}) = \frac{d'_{\text{inner}}}{d_{\text{inner}}} \text{Var}(\bm{y}_{\text{ffn}_i}).
\end{equation*}
\end{lemma}
\begin{proof}
The output components have zero mean (e.g., $\mathbb{E}[\bm{y}_{\text{ffn}_i}]=0$) because the second-layer weights $\bm{C}^{\text{ffn}}$ are initialized with zero mean and are independent of the first-layer activations. Thus, the variance is the expected squared value: $\text{Var}(\bm{y}_{\text{ffn}_i}) = \mathbb{E}[(\bm{y}_{\text{ffn}_i})^2]$.

Since the output components are identically distributed, we can relate the variance of a single component to the expected squared norm of the full output vector as
\begin{equation*}
\text{Var}(\bm{y}_{\text{ffn}_i}) = \frac{1}{d_{\text{model}}} \sum_{k=1}^{d_{\text{model}}} \text{Var}(\bm{y}_{\text{ffn}_k}) = \frac{1}{d_{\text{model}}} \mathbb{E}[\norm{\bm{y}_{\text{ffn}}}^2].
\end{equation*}
The same reasoning holds for $\bm{y}'_{\text{ffn}}$. From Thm.~\ref{thm:ffn_squared_scale_factor}, we know that $\mathbb{E} [\norm{\bm{y}'_{\text{ffn}}}^2] = \frac{d'_{\text{inner}}}{d_{\text{inner}}} \mathbb{E} [\norm{\bm{y}_{\text{ffn}}}^2]$. It follows that:
\begin{align*}
\text{Var}(\bm{y}'_{\text{ffn}_i}) &= \frac{1}{d_{\text{model}}} \mathbb{E}[\norm{\bm{y}'_{\text{ffn}}}^2] \\
&= \frac{d'_{\text{inner}}}{d_{\text{inner}}} \left( \frac{1}{d_{\text{model}}} \mathbb{E} [\norm{\bm{y}_{\text{ffn}}}^2] \right) \\
&= \frac{d'_{\text{inner}}}{d_{\text{inner}}} \text{Var}(\bm{y}_{\text{ffn}_i}). \qedhere
\end{align*}
\end{proof}

\begin{lemma} 
\label{lma:ffn_variance}
Let $\bm{x} \in \mathbb{R}^{d_{\text{model}}}$ be a random vector with i.i.d. components, zero mean, and unit variance. Let the output of a ffn block be $\bm{y}_{\text{ffn}} = \bm{C}^{\text{ffn}}\sigma_{\text{relu}}(\bm{W}^{\text{ffn}} \bm{x})$, where weights are initialized as $\bm{W}^{\text{ffn}} \sim \mathcal{N}(0, \frac{2}{d_{\text{model}}} \bm{I})$ and $\bm{C}^{\text{ffn}} \sim \mathcal{N}(0, \frac{1}{d_{\text{inner}}} \bm{I})$. Then for each component $i$ and large $d_{\text{model}}$, 
\begin{equation*}
\text{Var}([\bm{y}_{\text{ffn}}]_i) = 1.
\end{equation*}
\end{lemma}
\begin{proof}
Let $\bm{z} = \sigma_{\text{relu}}(\bm{W}^{\text{ffn}}\bm{x})$. The $i$-th output is $[\bm{y}_{\text{ffn}}]_i = \sum_{j} \bm{C}^{\text{ffn}}_{ij} \bm{z}_j$. Since the weights $\bm{C}^{\text{ffn}}_{ij}$ are i.i.d., zero-mean, and independent of the activations $\bm{z}_j$, we can compute the variance. By the Bienaymé formula, the variance of a sum of independent variables is the sum of their variances. Furthermore, for two independent variables $X$ and $Y$ where $\mathbb{E}[X]=0$, the variance of their product is $\text{Var}(XY) = \text{Var}(X)\mathbb{E}[Y^2]$. This gives
\begin{align*}
\text{Var}([\bm{y}_{\text{ffn}}]_i) &= \sum_{j=1}^{d_{\text{inner}}} \text{Var}(\bm{C}^{\text{ffn}}_{ij} \bm{z}_j) \\
&= \sum_{j=1}^{d_{\text{inner}}} \text{Var}(\bm{C}^{\text{ffn}}_{ij})\mathbb{E}[\bm{z}_j^2] \\
&= \sum_{j=1}^{d_{\text{inner}}} \frac{1}{d_{\text{inner}}} \mathbb{E}[\bm{z}_j^2] = \mathbb{E}[\bm{z}_1^2].
\end{align*}
The last equality holds as the activations $\bm{z}_j$ are identically distributed.

Let $\bm{a}_j = [\bm{W}^{\text{ffn}}\bm{x}]_j$ be the $j$-th pre-activation. Its variance is $\text{Var}(\sum_{k=1}^{d_{\text{model}}} \bm{W}^{\text{ffn}}_{jk}\bm{x}_k)$. Since the terms in the sum are independent, the variance is the sum of the variances. Moreover the weights $\bm{W}^{\text{ffn}}_{jk}$ and inputs $\bm{x}_k$ are independent and zero-mean, so the variance of their product is the product of their variances, $\text{Var}(\bm{W}^{\text{ffn}}_{jk}\bm{x}_k) = \text{Var}(\bm{W}^{\text{ffn}}_{jk})\text{Var}(\bm{x}_k)$. We now have
\begin{align*}
\text{Var}(\bm{a}_j) &= \sum_{k=1}^{d_{\text{model}}} \text{Var}(\bm{W}^{\text{ffn}}_{jk}\bm{x}_k) \\
&= \sum_{k=1}^{d_{\text{model}}} \text{Var}(\bm{W}^{\text{ffn}}_{jk})\text{Var}(\bm{x}_k) \\
&= \sum_{k=1}^{d_{\text{model}}} \frac{2}{d_{\text{model}}} \cdot 1 = 2.
\end{align*}
For large $d_{\text{model}}$, the Central Limit Theorem implies $\bm{a}_j \sim \mathcal{N}(0, 2)$. For a zero-mean Gaussian $v \sim \mathcal{N}(0, \sigma^2)$, the second moment of its ReLU activation is $\mathbb{E}[(\sigma_{\text{relu}}(v))^2] = \frac{1}{2}\text{Var}(v)$. Thus,
\begin{equation*}
\mathbb{E}[\bm{z}_1^2] = \mathbb{E}[(\sigma_{\text{relu}}(\bm{a}_1))^2] = \frac{1}{2}\text{Var}(\bm{a}_1) = \frac{2}{2} = 1.
\end{equation*}
Substituting this back yields $\text{Var}([\bm{y}_{\text{ffn}}]_i) = 1$.
\end{proof}

\begin{theorem}[FFN Scale Factor]
Let $\bm{x} \in \mathbb{R}^{d_{\text{model}}}$ be a random vector with i.i.d. components, zero mean, and unit variance. Given a two layer ffn as defined in Section \ref{sec:notation}
\begin{equation*}
\bm{y}_{\text{ffn}} = \bm{C}^{\text{ffn}}\sigma_{\text{relu}}(\bm{W}^{\text{ffn}} \bm{x})
\end{equation*}
with inner dimension $d_{\text{inner}}$, changing the inner dimension to $d'_{\text{inner}} \in \{ 0, 1, 2, ..., d_{\text{inner}} \}$ without resampling the
weights (i.e., taking a subset of the parameters) scales the expected output norm as follows with high probability:
\begin{equation*}
\mathbb{E} [\norm{\bm{y}'_{\text{ffn}}}] = \sqrt{\frac{d'_{\text{inner}}}{d_{\text{inner}}}} \mathbb{E} \norm{\bm{y}_{\text{ffn}}}]
\end{equation*}
for sufficiently large $d_{\text{model}}$.
\label{thm:ffn_scale_factor}
\end{theorem}

\begin{proof}
From Theorem \ref{thm:ffn_squared_scale_factor}, we have an exact relation for the expected squared norms:
\begin{equation*}
\mathbb{E} [\norm{\bm{y}'_{\text{ffn}}}^2] = \frac{d'_{\text{inner}}}{d_{\text{inner}}} \mathbb{E} [\norm{\bm{y}_{\text{ffn}}}^2]
\end{equation*}

Taking the square root of both sides gives:
\begin{equation}
\sqrt{\mathbb{E} [\norm{\bm{y}'_{\text{ffn}}}^2]} = \sqrt{\frac{d'_{\text{inner}}}{d_{\text{inner}}}} \sqrt{\mathbb{E} [\norm{\bm{y}_{\text{ffn}}}^2]}
\label{eq:proof_sqrt_relation}
\end{equation}

We argue that the approximation $\mathbb{E}[\norm{\cdot}] = \sqrt{\mathbb{E}[\norm{\cdot}^2]}$ holds for high-dimensional random vectors like $\bm{y}_{\text{ffn}}$ and $\bm{y}'_{\text{ffn}}$. For a vector $\bm{z} \in \mathbb{R}^{d_{\text{model}}}$, the squared norm $\norm{\bm{z}}^2 = \sum_{i=1}^{d_{\text{model}}} z_i^2$ is a sum of a large number of random variables. As established in Lemma \ref{lma:ffn_variance} and \ref{lma:ffn_subset_out}, the components of $\bm{y}_{\text{ffn}}$ and $\bm{y}'_{\text{ffn}}$ have well-behaved variances. Moreover, as an intermediate result we have Uncorrelated Output Components from \ref{lma:ffn_variance}. So, for a sufficiently large $d_{\text{model}}$, Chebyshev's Weak Law of Large Numbers \citep{Taboga2021ChebyshevWLLN} implies that the sum $\norm{\cdot}^2$ will be sharply concentrated around its expected value, $\mathbb{E}[\norm{\cdot}^2]$.

When a random variable is highly concentrated, its value is very close to its mean with high probability. By the continuity of the square-root function, if $\norm{\cdot}^2$ is concentrated around $\mathbb{E}[\norm{\cdot}^2]$, then its square root, $\norm{\cdot}|$, must be concentrated around $\sqrt{\mathbb{E}[\norm{\cdot}^2]}$. The expectation of a tightly concentrated random variable is nearly equal to the value it is concentrated around. Therefore, for large $d_{\text{model}}$, we have:
\begin{align*}
\mathbb{E} [\norm{\bm{y}_{\text{ffn}}}] &\approx \sqrt{\mathbb{E} [\norm{\bm{y}_{\text{ffn}}}^2]} \\
\mathbb{E} [\norm{\bm{y}'_{\text{ffn}}}] &\approx \sqrt{\mathbb{E} [\norm{\bm{y}'_{\text{ffn}}}^2]}
\end{align*}
Substituting these high-probability approximations into Equation \eqref{eq:proof_sqrt_relation} directly yields the desired result:
\begin{equation*}
\mathbb{E} [\norm{\bm{y}'_{\text{ffn}}}] = \sqrt{\frac{d'_{\text{inner}}}{d_{\text{inner}}}} \mathbb{E} [\norm{\bm{y}_{\text{ffn}}}] \qedhere
\end{equation*}
\end{proof}

\begin{theorem}
\label{thm:attn_squared_scale_factor}
Given a Multi-Head Attention layer with output $\bm{Y}_{\text{attn}} = \text{Concat}(\text{head}_1, \dots, \text{head}_H) \bm{C}^{\text{attn}}$, reducing the number of heads to $H' \leq H$ by selecting a subset scales the expected squared norm of the output rows as:
\begin{equation*}
\mathbb{E} [\norm{[\bm{Y}'_{\text{attn}}]_i}^2] = \frac{H'}{H} \mathbb{E} [\norm{[\bm{Y}_{\text{attn}}]_i}^2].
\end{equation*}
\end{theorem}
\begin{proof}
Let $\bm{A} = \text{Concat}(\text{head}_1, \dots, \text{head}_H)$. The projection matrix $\bm{C}^{\text{attn}}$ is initialized independently from $\bm{A}$ with zero-mean entries, so $\mathbb{E}[[\bm{Y}_{\text{attn}}]_i] = 0$. The expected squared norm of an output row is:
\begin{align*}
\mathbb{E} [ \norm{[\bm{Y}_{\text{attn}}]_i}^2 ] &= \sum_{j=1}^{d_{\text{model}}} \text{Var}\left(\sum_{k=1}^{H d_{\text{head}}} \bm{A}_{ik} \bm{C}^{\text{attn}}_{kj}\right) \\
&= \sum_{j} \sum_{k} \text{Var}(\bm{A}_{ik} \bm{C}^{\text{attn}}_{kj}) \\
&= \sum_j \sum_k \mathbb{E}[\bm{A}_{ik}^2]\text{Var}(\bm{C}^{\text{attn}}_{kj}).
\end{align*}
The second equality holds because the entries of $\bm{C}^{\text{attn}}$ are i.i.d. and zero-mean, causing covariance terms from the Bienaymé formula to vanish. The third holds due to the independence of $\bm{A}$ and $\bm{C}^{\text{attn}}$. Let $\sigma_C^2 = \text{Var}(\bm{C}^{\text{attn}}_{kj})$. As all heads are statistically identical due to i.i.d. initialization, we can rewrite the sum over the concatenated dimension $k$ as a factor of $H$:
\begin{equation*}
\mathbb{E} [ \norm{[\bm{Y}_{\text{attn}}]_i}^2 ] = H \left( \sigma_C^2 \sum_{j=1}^{d_{\text{model}}} \sum_{l=1}^{d_{\text{head}}} \mathbb{E}[\bm{A}_{i,l}^2] \right),
\end{equation*}
where the sum over $l$ is across a single head's dimensions. The term in parentheses is an invariant with respect to the number of heads. For a layer with $H'$ heads, the leading factor is simply $H'$. The theorem follows by taking the ratio.
\end{proof}

\begin{lemma}\label{lma:attn_subset_out}
Let $\bm{X} \in \mathbb{R}^{N \times d_{\text{model}}} \sim \mathcal{N} (0, \bm{I})$. For a Multi-Head Attention layer with $H$ heads, $\bm{Y}_{\text{attn}} = \text{Concat}(\text{head}_1, \dots, \text{head}_H) \bm{C}^{\text{attn}}$, reducing the number of heads to $H' \leq H$ by taking a subset scales the output variance as
\begin{equation*}
\text{Var}([\bm{Y}'_{\text{attn}}]_{ij}) = \frac{H'}{H} \text{Var}([\bm{Y}_{\text{attn}}]_{ij}).
\end{equation*}
\end{lemma}
\begin{proof}
At initialization, we approximate the softmax as attending to all tokens equally \citep{Kedia2024Transformers, Chi2023Latent, Xiong2020LayerNorm}. Thus, $\text{head}_h \approx \frac{1}{N}\bm{1}_{N} \bm{1}_N^{\top} \bm{X} \bm{W}^{V}_h$. Let $\bar{\bm{X}} = \frac{1}{N}\bm{1}_{N} \bm{1}_N^{\top} \bm{X}$. Since $\bm{X}_{lk} \sim \mathcal{N}(0, 1)$ are i.i.d., the elements of $\bar{\bm{X}}$ have $\mathbb{E}[\bar{\bm{x}}_{ik}] = 0$ and $\text{Var}(\bar{\bm{x}}_{ik}) = 1/N$.

The elements of a single head, $[\text{head}_{h}]_{ij} = \sum_{k=1}^{d_{\text{model}}} \bar{\bm{x}}_{ik} [\bm{W}_h^{V}]_{kj}$, are zero-mean. By our initialization, the weights $[\bm{W}_h^{V}]_{kj}$ are i.i.d. with variance $\sigma_V^2$ and are independent of $\bar{\bm{X}}$. So, the variance $\text{Var}([\text{head}_{h}]_{ij})$ is
\begin{equation*}
\sum_{k=1}^{d_{\text{model}}} \text{Var}(\bar{\bm{x}}_{ik}) \text{Var}([\bm{W}_h^{V}]_{kj}) = \frac{d_{\text{model}}}{N}\sigma_V^2.
\end{equation*}
Let $\bm{Z} = \text{Concat}(\text{head}_1, \dots, \text{head}_H)$, and let $\sigma_C^2$ be the variance of $\bm{C}^{\text{attn}}_{kj}$. Notice the elements of the output layer $[\bm{Y}_{\text{attn}}]_{ij} = \sum_{k=1}^{H d_{\text{head}}} \bm{Z}_{ik} \bm{C}^{\text{attn}}_{kj}$ are zero-mean. Since each head is independent, the variance is
\begin{align}
\text{Var}([\bm{Y}_{\text{attn}}]_{ij}) &= \sum_{k=1}^{H d_{\text{head}}} \text{Var}(\bm{Z}_{ik})\text{Var}(\bm{C}^{\text{attn}}_{kj}) \nonumber \\
&= (H d_{\text{head}}) \left(\frac{d_{\text{model}}}{N}\sigma_V^2\right) \sigma_C^2. \label{eq:attn_out_var}
\end{align}
For a layer with a subset of $H'$ heads, the derivation is identical, replacing $H$ with $H'$. Taking the ratio with Eq. \eqref{eq:attn_out_var} gives the result.
\end{proof}

\begin{lemma}
Let $\bm{X} \in \mathbb{R}^{N \times d_{\text{model}}} \sim \mathcal{N} (0, \bm{I})$. Given a Multi-Head Attention layer as defined in Section \ref{sec:notation}. Then for each component of the attention layer output $i \in \{0, 1, \dots , N - 1\}, \ j \in \{ 0, 1, \dots, d_{\text{model}} \}$,
\begin{equation*}
\bm{Var} \left( [\bm{Y}_{\text{attn}}]_{ij} \right) = \frac{1}{N}
\end{equation*}
for sufficiently large $d_{\text{model}}$.
\label{lma:attn_variance}
\end{lemma}

\begin{proof}
By Equation \ref{eq:attn_out_var} of Lemma \ref{lma:attn_subset_out}
\begin{align*}
\text{Var} \left( \left[ \bm{Y}_{\text{attn}} \right]_{ij} \right) &= H d_{\text{head}} \frac{d_{\text{model}}}{N} \sigma_x^2 \sigma_V^2 \sigma_C^2 \\
&= H d_{\text{head}} \frac{d_{\text{model}}}{N} \left( 1 \right) \left( \frac{1}{d_{\text{model}}} \right) \left( \frac{1}{H d_{\text{head}}} \right) \\
&= \frac{1}{N} \qedhere
\end{align*}
\end{proof}

\begin{theorem}[Attention Scale Factor]
\label{thm:attn_scale_factor}
Let $\bm{X} \in \mathbb{R}^{N \times d_{\text{model}}} \sim \mathcal{N} (0, \bm{I})$. For a Multi-Head Attention layer with $H$ heads, $\bm{Y}_{\text{attn}} = \text{Concat}(\text{head}_1, \dots, \text{head}_H) \bm{C}^{\text{attn}}$, reducing the number of heads to $H' \leq H$ by taking a subset scales the expected output norm as
\begin{equation*}
\mathbb{E} [\norm{[\bm{Y}'_{\text{attn}}]_i}] = \sqrt{\frac{H'}{H}} \mathbb{E} [\norm{[\bm{Y}_{\text{attn}}]_i}].
\end{equation*}
\end{theorem}
\begin{proof}
We argue similarly to Theorem \ref{thm:ffn_scale_factor}. From Theorem \ref{thm:attn_squared_scale_factor}, the expected squared norms are exactly related by $\mathbb{E} [\norm{[\bm{Y}'_{\text{attn}}]_i}^2] = \frac{H'}{H} \mathbb{E} [\norm{[\bm{Y}_{\text{attn}}]_i}^2]$. Taking the square root gives
\begin{equation}
\sqrt{\mathbb{E} [\norm{[\bm{Y}'_{\text{attn}}]_i}^2]} = \sqrt{\frac{H'}{H}} \sqrt{\mathbb{E} [\norm{[\bm{Y}_{\text{attn}}]_i}^2]}.
\label{eq:attn_proof_sqrt_relation_trimmed}
\end{equation}
For large $d_{\text{model}}$, the squared norm $\norm{[\bm{Y}_{\text{attn}}]_i}^2$ is a sum over many uncorrelated components with finite variance (Lemmas \ref{lma:attn_variance} and \ref{lma:attn_subset_out}). By the Law of Large Numbers, the squared norm concentrates around its mean. The continuous mapping theorem then implies that the norm concentrates around the square root of the mean, justifying the approximation $\mathbb{E}[\norm{\cdot}] \approx \sqrt{\mathbb{E}[\norm{\cdot}^2]}$. Substituting this into Eq. \eqref{eq:attn_proof_sqrt_relation_trimmed} yields the result.
\end{proof}


\section{Additional Experiments}
\label{sec:additional_experiments}

\subsection{Additional System Stability Experiments}
\label{sec:additional_system_stability}

\begin{figure}[htbp]                                      
\centering
\includegraphics[width=\columnwidth]{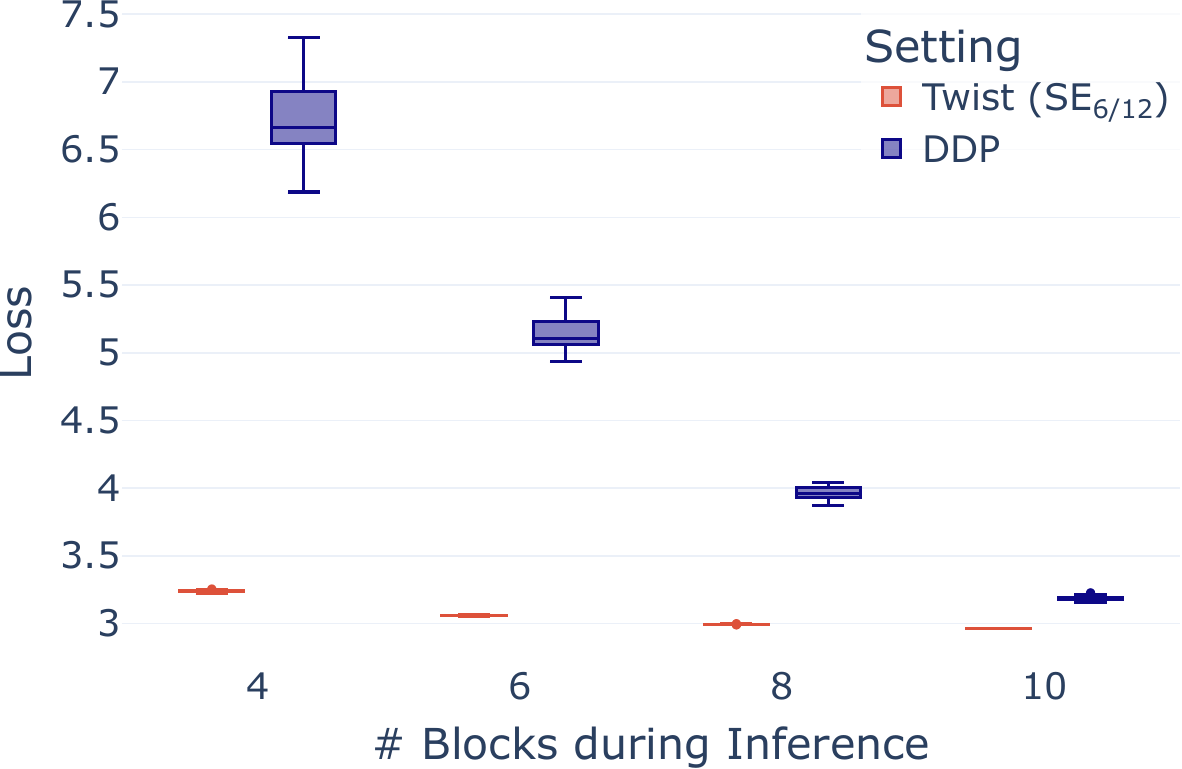} 
\caption{Distribution of eval loss for randomly generated subnets in the \texttt{both} configuration. The distributions for \TwIST{} ($SE_{6/12}$) are compared against a \texttt{DDP} baseline across various subnet ratios. The $SE_{6/12}$ variant of Twist is presented for a direct comparison, as both this method and \texttt{DDP} involve only a single training pass.}
\label{fig:subnet_loss_distribution_both}
\end{figure}

Figure~\ref{fig:subnet_loss_distribution_both} presents a detailed breakdown of the evaluation loss distributions for the \TwIST{} ($SE_{6/12}$) method compared against the \texttt{DDP} baseline. This analysis was performed on randomly generated subnets, varying the number of active blocks during inference (4, 6, 8, and 10). The results demonstrate the significant stability and superior performance of the \TwIST{} method. Across all tested subnet configurations, \TwIST{} maintains a consistent and low evaluation loss (approx. 3.1-3.2) with negligible variance. In contrast, the \texttt{DDP} baseline exhibits both substantially higher loss and high variance, particularly with fewer active blocks. The median loss for \texttt{DDP} is highest at 4 blocks (approx. 6.8) and progressively improves as the number of blocks increases, eventually approaching the loss of \TwIST{} at 10 blocks (approx. 3.2). This comparison highlights that \TwIST{} provides robust performance regardless of the active subnet width, a significant advantage over the more variable and width-dependent \texttt{DDP} baseline.

\begin{figure}[htbp]                                      
\centering
\includegraphics[width=\columnwidth]{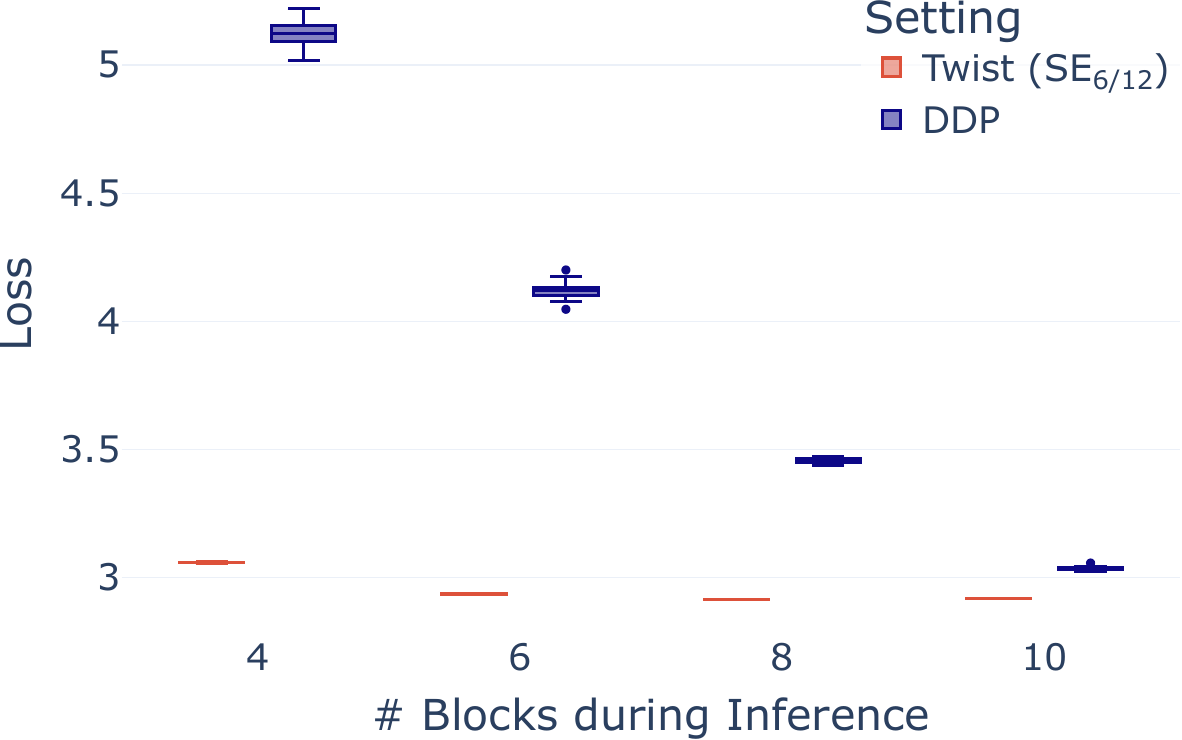} 
\caption{Distribution of eval loss for randomly generated subnets in the \texttt{ffn} configuration. The distributions for \TwIST{} ($SE_{6/12}$) are compared against a \texttt{DDP} baseline across various subnet ratios. The $SE_{6/12}$ variant of Twist is presented for a direct comparison, as both this method and \texttt{DDP} involve only a single training pass.}
\label{fig:subnet_loss_distribution_ffn}
\end{figure}

This finding is further reinforced by the analysis of the \texttt{ffn} configuration, shown in Figure~\ref{fig:subnet_loss_distribution_ffn}. The results exhibit a virtually identical pattern: the TWIST ($SE_{6 / 12}$) method again achieves a consistently low loss (approx. 3.0) with near-zero variance, irrespective of the number of active blocks. The DDP baseline, while demonstrating lower overall losses in this configuration (e.g., a median loss of approx. 5.1 at 4 blocks, compared to ~6.8 in Figure~\ref{fig:subnet_loss_distribution_both}), displays the same critical dependency on network width. Its loss and variance are highest at the leanest width and only converge with \TwIST's superior performance when the full 10 blocks are utilized. Taken together, Figures~\ref{fig:subnet_loss_distribution_both} and \ref{fig:subnet_loss_distribution_ffn} provide strong evidence that the \TwIST{} methodology produces subnets that are robustly performant at arbitrary widths, a key advantage over the width-sensitive \texttt{DDP} baseline.

\subsection{Additional Architectural Robustness Experiments}
\label{sec:additional_architectural_robustness}

\begin{figure}[htb]
\centering
\includegraphics[width=0.75\columnwidth]{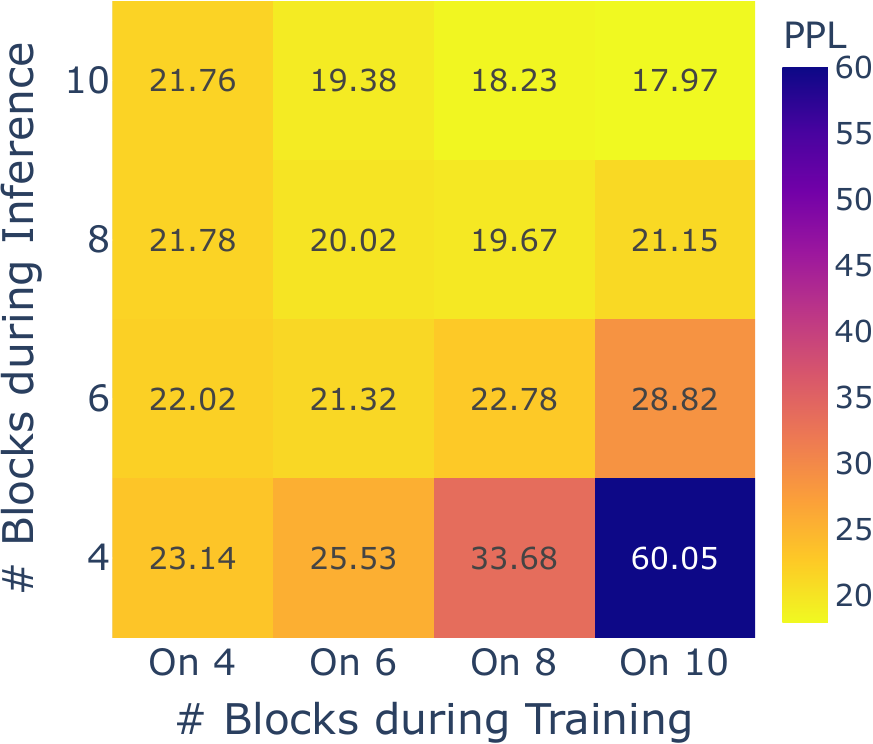} 
\caption{Heatmap of subnet robustness for the \texttt{both} setting. Brighter colors (yellow) signify lower PPL (better performance), and darker colors (blue) signify higher PPL.}
\label{fig:robustness_heatmap_both}
\end{figure}

Figures~\ref{fig:robustness_heatmap_both} and \ref{fig:robustness_heatmap_ffn} further investigate the issue of subnet robustness by analyzing the mismatch between training and inference depths. Figure~\ref{fig:robustness_heatmap_both} provides a perplexity (PPL) heatmap for the \texttt{both} setting, where lower PPL (brighter yellow) is better. The x-axis represents the number of blocks used during training (e.g., ``On 4''), and the y-axis represents the number of blocks used during inference. The plot reveals a severe performance degradation when models trained on wide subnets are inferred on lean ones. For instance, the model trained "On 10" blocks, while performing well at 10-block inference (17.97 PPL), experiences a catastrophic failure at 4-block inference (60.05 PPL). Conversely, the model trained ``On 4'' blocks shows remarkable robustness, maintaining a stable PPL (between 21.76 and 23.14) across all inference depths.

\begin{figure}[hbt]
\centering
\includegraphics[width=0.75\columnwidth]{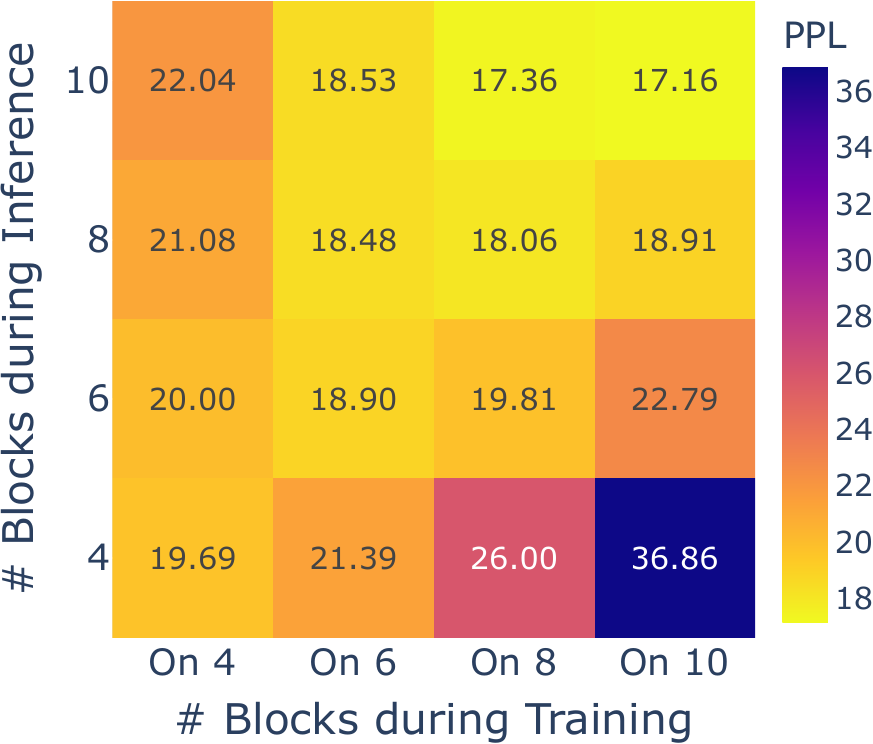} 
\caption{Heatmap of subnet robustness for the \texttt{ffn} setting. Brighter colors (yellow) signify lower PPL (better performance), and darker colors (blue) signify higher PPL.}
\label{fig:robustness_heatmap_ffn}
\end{figure}

Figure~\ref{fig:robustness_heatmap_ffn} confirms this exact finding in the \texttt{ffn} configuration. The identical trend is observed: models trained on wide subnets are ``brittle'' and not robust to leaner inference. The model trained ``On 10'' blocks sees its PPL degrade from 17.16 (at 10-block inference) to 36.86 (at 4-block inference). In stark contrast, the model trained ``On 4'' blocks again proves the more robust, with its PPL remaining in a tight, low range (19.69 to 22.04) regardless of the number of blocks used at inference. Both heatmaps strongly indicate that considering target sparsity from the start of training is crucial for building models that are robust to dynamic inference depths, whereas a traditional training strategy leads to significant performance collapse on smaller subnets.


\section{Computing memory and communication}
\label{sec:memory_and_communication}

We present formula used for computing the number of parameters in GPT-2 style transformer. Let $N_{\text{embd}} = (N_\text{vocab}) (d_\text{model})$ be the number of parameters in the embedding where $N_{\text{vocab}}$ is the size of the vocabulary. Let $N_{\text{attn}} = 4 (d_{\text{attn}}) (d_{\text{model}}) + 3(d_{\text{attn}}) + d_{\text{model}}$ be the number of parameters in an attention layer where $d_{\text{attn}}$ is computed as the number of heads in that layer multiplied by the head dimension. The first term is the total number of parameters in the $\bm{W}^Q, \bm{W}^K, \bm{W}^V, \bm{C}^{\text{attn}}$ weights. The following two terms are the number of parameters in the biases for the $\{ \bm{W}^Q, \bm{W}^K, \bm{W}^V \}$ and $\bm{C}^{\text{attn}}$ weights respectively. Let $N_{\text{ffn}} = 2 (d_{\text{ffn}})(d_{\text{model}}) + d_{\text{ffn}} + d_{\text{model}}$ be the number of parameters in a feedforward layer. Let $N_{\text{ln}} = 2 d_{\text{model}}$ be the number of parameters in a layer norm. Then, the number of parameters in a transformer layer is $N_{\text{layer}}^l = N_{\text{ln}} + \alpha_l N_{\text{attn}} + \beta_l N_{\text{ffn}}$ where $\alpha$ is the sparsity of the attention module and $\beta$ is the sparsity of the feedforward module. Moreoever, the total number of parameters in the transformer is
\begin{equation*}
N_{\text{model}} = N_{\text{embed}} + \sum_{l} N_{\text{layer}}^l + N_{\text{proj}}
\end{equation*}
where $N_{\text{proj}}$ is the number of parameters in the final transformer projection. This depends on the task. For text generation, $N_{\text{proj}} = 0$ assuming tied weights. For text classification, $N_{\text{proj}} = (d_{\text{model}})(N_{\text{labels}})$ where $N_{\text{labels}}$ is the number of categories. 

The number of parameters is roughly proportional to the amount of space a model takes up on a hardware accelerator both at model load time (where dtype determines bytes per param) and at train time (where it is about four times the model size when using Adam) \citep{HuggingFace_ModelEstimator_2025}. Since every parameter is both sent to and from a worker in a communication round, the total communication cost is proportional to the number of communication rounds multiplied by the number of parameters.


\end{document}